\documentclass[journal,comsoc]{IEEEtran}

\usepackage{amsmath,amssymb,amsbsy,amsthm,mathrsfs,bbm}
\usepackage{mathtools} 
\usepackage{graphicx,epsfig}
\usepackage{multirow,array,tabularx,makecell}
\usepackage{color}
\usepackage{cite}
\usepackage{caption}
\allowdisplaybreaks
\usepackage{subcaption}
\usepackage{algorithm}         
\usepackage{algpseudocode}     
\newcommand{\st}{{\mathrm{s.t.}}}
\DeclareMathAlphabet{\mathpzc}{OT1}{pzc}{m}{it}

\newtheorem{theorem}{\textbf{\textsc{Theorem}}}

\begin{document}
\title{Deep Learning-Driven Friendly Jamming for Secure Multicarrier ISAC Under Channel Uncertainty}
\author{
\IEEEauthorblockN{Bui Minh Tuan, Van-Dinh Nguyen, Diep N. Nguyen, Nguyen Linh Trung, \\   Nguyen Van Huynh, Dinh Thai Hoang, Marwan Krunz,  and Eryk Dutkiewicz}

\thanks{The work of V.-D. Nguyen was supported by the Australia-Vietnam Strategic Technologies Centre (AVSTC) and VUNI.GREEN-X.RISE.AY25-27.03.}
\thanks{Bui Minh Tuan, Diep N. Nguyen, Dinh Thai Hoang, and Eryk Dutkiewicz are with the School of Electrical and Data Engineering, University of Technology Sydney, NSW 2007, Australia (bui.m.tuan@student.uts.edu.au; \{diep.nguyen, hoang.dinh, eryk.dutkiewicz\} @uts.edu.au).}

\thanks{V.-D. Nguyen is with the School of Computer Science and Statistics, Trinity College Dublin, Dublin 2, D02PN40, Ireland  (e-mail: dinh.nguyen@tcd.ie), and was with Smart Green Transformation Center (GREEN-X), VinUniversity, Hanoi, Vietnam. V.-D. Nguyen carried
out his contribution while affiliated with VinUniversity.}
\thanks{Nguyen Van Huynh is with the Department of Electrical Engineering and
Electronics, University of Liverpool, Liverpool, L69 3GJ, United Kingdom
(e-mail: huynh.nguyen@liverpool.ac.uk).}
\thanks{Marwan Krunz is with the Department of Electrical and Computer Engineering, the University of Arizona, Tucson, AZ, USA (krunz@arizona.edu).}
\thanks{Nguyen Linh Trung is with the Faculty of Electronics and Telecommunications, University of Engineering and Technology, Vietnam National University, Hanoi, Vietnam (linhtrung@vnu.edu.vn).}
\vspace{-15pt}
}

\maketitle

\begin{abstract}
Integrated sensing and communication (ISAC) systems promise efficient spectrum utilization by jointly supporting radar sensing and wireless communication. This paper presents a deep learning-driven framework for enhancing physical-layer security in multicarrier ISAC systems under imperfect channel state information (CSI) and in the presence of unknown eavesdropper (Eve) locations. Unlike conventional ISAC-based friendly jamming (FJ) approaches that require Eve's CSI or precise angle-of-arrival (AoA) estimates, our method exploits radar echo feedback to guide directional jamming without explicit Eve's information. To enhance robustness to radar sensing uncertainty, we propose a radar-aware neural network that jointly optimizes beamforming and jamming by integrating a novel nonparametric Fisher Information Matrix (FIM) estimator based on f-divergence. The jamming design satisfies the Cramér–Rao lower bound (CRLB) constraints even in the presence of noisy AoA. For efficient implementation, we introduce a quantized tensor train-based encoder that reduces the model size by more than 100 times with negligible performance loss. We also integrate a non-overlapping secure scheme into the proposed framework, in which specific sub-bands can be dedicated solely to communication. {\color{black}
Extensive simulations demonstrate that the proposed solution achieves significant improvements in secrecy rate, reduced block error rate (BLER), and strong robustness against CSI uncertainty and angular estimation errors, underscoring the effectiveness of the proposed deep learning--driven friendly jamming framework under practical ISAC impairments.
}
\end{abstract}
	
\section{Introduction}
Integrated sensing and communication (ISAC) is a promising technology for next-generation wireless systems, enabling base stations (BSs) to transmit data and sense the environment with a unified waveform~\cite{liu2020joint, nguyen2025performance, su2023sensing, nasir2024joint, zhang2024secure, jiang2023physical, su2020secure, xu2022robust}. By sharing spectrum and hardware, ISAC significantly improves efficiency and supports advanced applications such as autonomous vehicles, unmanned aerial systems, and intelligent infrastructure. However, ISAC is highly vulnerable to eavesdropping. Unlike conventional systems, its signals contain both user data and environmental information, allowing adversaries to simultaneously intercept private communications and sensing insights~\cite{wei2022toward}. This dual exposure amplifies security risks and demands robust protection mechanisms.

{\color{black}
\subsection{Related Work}
Physical-layer security (PLS) techniques such as friendly jamming (FJ) have been widely studied \cite{su2022secure,jiang2023physical}. Yet most assume perfect CSI or sensing information, which is unrealistic given the passive nature of eavesdroppers (Eves). Moreover, many works optimize secrecy rate alone while overlooking practical metrics such as block error rate (BLER) and the complexity of multiuser multicarrier systems. In this work, we exploit ISAC's sensing capability to estimate Eve's location and guide FJ design, thereby degrading Eve's reception. The probing signal combines structured orthogonal frequency-division multiplexing (OFDM) pilots, supporting both channel estimation and radar sensing, with a jamming signal that lowers Eve's signal-to-noise ratio (SNR). This introduces dual uncertainties from angle-of-arrival (AoA) estimation errors and imperfect CSI, which complicate the joint design of secure transmission and FJ. To our knowledge, this dual uncertainty has not been addressed in prior ISAC security studies. To overcome these challenges, we propose a deep learning framework that operates without perfect CSI. Furthermore, we introduce a quantized tensor train (TT-Q) decomposition model to reduce complexity, enabling a lightweight and scalable security solution for ISAC systems.
}

Various techniques have been proposed to strengthen PLS in ISAC systems. Early work on semantic security extends beyond conventional secrecy metrics by addressing higher-level confidentiality requirements \cite{yang2024secure}, but such methods are mostly limited to narrowband settings and scale poorly to multicarrier ISAC. Sensing-assisted approaches \cite{su2023sensing,nasir2024joint} use radar-derived angle estimates to localize Eves. Although promising, they depend on highly accurate AoA estimation, which is unreliable in mobile, obstructed, or fading environments.

Convex optimization-based techniques provide theoretical guarantees using tools such as semidefinite relaxation (SDR), alternating optimization, and fractional programming \cite{liu2024artificial, bazzi2023secure}. However, they typically assume perfect CSI, and their performance deteriorates under real-time constraints or dynamic propagation. Even advanced formulations that incorporate secrecy rate (SR), signal-to-interference-plus-noise ratio, or mean squared error \cite{zhang2024secure, jiang2023physical, su2020secure, xu2022robust} generally rely on bounded or perfect Eve's CSI, which limits their practicality. To address channel uncertainty, robust beamforming schemes have been proposed under bounded or statistical CSI errors \cite{hou2024optimal, ren2023robust, liu2022outage}. These methods employ tools such as the S-procedure, Schur complement, and successive convex approximation (SCA), improving resilience to legitimate user CSI imperfections. Nevertheless, their performance still degrades in high-mobility or multiuser scenarios.

Another active line of research is sensing-assisted secure beamforming, where radar-derived angular information is used to partition spatial regions and infer Eve's risk zones \cite{ren2023robust,liu2024artificial,bazzi2023secure}.  The region partitioning and adaptive secure beamforming \cite{cao2024sensing, su2023sensing, nasir2024joint} integrated sensing with optimization tools like polyblock approximation and fixed-point iteration. While these approaches mitigate the absence of Eve's CSI, they often assume static environments and highly accurate sensing, limiting their effectiveness in dynamic or noisy conditions.

Following this, other studies explored joint waveform and beamforming design to balance secrecy and sensing objectives. For example, \cite{deligiannis2018secrecy, yang2023secure} jointly optimized secrecy rate, mutual information, and estimation accuracy, such as Cramer-Rao lower bound CRLB, under power and sensing constraints using fractional programming and dual decomposition. Similarly, secure waveform designs in \cite{su2022secure, gunlu2023secure} aimed to suppress eavesdropping while maintaining radar performance. However, these methods still relied on perfect CSI and computationally intensive algorithms, which hinder scalability and real-time deployment. Emerging architectures, such as full-duplex ISAC and secrecy-distortion frameworks, added new opportunities but also increased system complexity. For instance, \cite{liu2024artificial, bazzi2023secure} investigated robust beamforming for full-duplex multi-input and multi-output (MIMO) under CSI uncertainty, while \cite{zou2024securing} proposed a distortion-aware secrecy model with output feedback. These works broaden the design space but still struggle to balance CSI uncertainty with practical complexity.

More recently, learning-based approaches have begun to reshape secure ISAC. Deep neural networks (DNNs) have been applied to generate secure precoders under imperfect CSI \cite{li2024deep}, and predictive beamforming has been explored for mobile aerial threats \cite{al2025predictive}. While promising, these methods often require large models, lack online adaptability, and remain difficult to deploy on 
resource-constrained platforms.
In many learning-based designs, heterogeneous inputs such as OFDM waveforms, CSI, and AoA estimates are embedded into high-dimensional vectors, particularly when processed by fully connected (FC) architectures for precoding or FJ design. However, such models face scalability challenges in multiuser MIMO scenarios, where dimensionality increases with antennas, subcarriers, and time samples. On the other hand, FC networks are prone to overfitting and impose high computational and memory costs, making them unsuitable for real-time or edge applications.

{\color{black}
Tensor train decomposition, originally introduced as an efficient representation for high dimensional tensors \cite{oseledets2011tensor}, has been widely adopted for neural network compression by factorizing large layers into compact tensor cores \cite{novikov2015tensorizing}. To further reduce memory footprint and inference complexity, quantized tensor representations have been proposed, where tensor train cores are stored using low bit precision or shared codebooks. Recent works such as QTTNet \cite{sun2020deep} demonstrate that quantized tensor train networks achieve substantial model compression with minimal performance degradation. Related studies have also explored canonical polyadic decomposition and tensor quantization for efficient learning and joint communication sensing parameter estimation in ISAC systems \cite{chang2022tensor, cheng2023nested, zhang2024integrated, du2024dl}. These advances motivate the adoption of a quantized tensor train based encoder in this work to enable lightweight and scalable secure ISAC implementations.}

{\color{black}
While ISAC systems inherently support both sensing and communication functionalities, this paper does not aim to protect sensing information itself. Rather, it focuses on leveraging sensing-derived information to enhance the physical-layer security of communication. Existing works on secure ISAC can be broadly categorized into approaches that jointly optimize sensing and communication performance and those that exploit sensing feedback to improve communication confidentiality. This paper belongs to the latter category, where radar sensing is used to assist friendly jamming and beamforming design for secure communication under imperfect CSI and unknown eavesdropper locations.}

{\color{black}

\begin{table*}[h!]
\centering
\caption{{\color{black}
Comparison between representative secure ISAC approaches and the proposed method.}}
\label{tab:comparison2}
\resizebox{\textwidth}{!}{%
\begin{tabular}{|c|c|c|c|c|c|c|}
\hline
\textbf{Work} &
\textbf{ISAC} &
\textbf{Eve CSI} &
\textbf{Sensing Role} &
\textbf{Uncertainty Handling} &
\textbf{MC} &
\textbf{Method} \\
\hline
{\cite{su2023sensing, nasir2024joint,liu2024artificial}} & Yes & No (AoA only) &
Region/angle-aided design &
Sensitive to AoA errors; limited robustness &
No &
Optimization \\
\hline
{\cite{su2022secure, deligiannis2018secrecy,yang2023secure} } & Yes & Yes / bounded &
Joint waveform/beamforming &
Typically assumes perfect or bounded CSI; limited angular uncertainty treatment &
Limited &
SDR/FP/AO \\
\hline
{[\cite{hou2024optimal, ren2023robust, liu2022outage}} & No & Bounded/statistical &
Not considered &
Robust to user CSI errors; not ISAC sensing-guided &
Yes &
Robust opt. \\
\hline
{\cite{li2024deep, al2025predictive}} & Partial & Imperfect CSI &
Learning/prediction &
Data-driven CSI robustness; limited CRLB-guaranteed sensing constraints &
Yes &
Deep learning \\
\hline
\textbf{This work} & Yes & \textbf{No} &
Radar-echo-guided jamming &
\textbf{CSI \& angular uncertainty} with \textbf{CRLB constraints} &
\textbf{Yes} &
\textbf{Learning + TT-Q} \\
\hline
\end{tabular}}
\end{table*}

Table \ref{tab:comparison2} provides a structured comparison between representative secure ISAC and sensing-assisted physical-layer security approaches and the proposed framework.
}

\subsection{Motivation and Main Contributions} 
{\color{black}
In this work, we aim to enhance the physical-layer security of wireless communication by exploiting the sensing capability of ISAC systems. Radar sensing is used as an enabling tool to infer spatial and angular information about potential eavesdroppers and to guide friendly jamming and beamforming design. The security objective is defined exclusively for communication and is quantified using standard physical-layer security metrics, including secrecy rate and BLER at legitimate users. Sensing is not treated as a confidential service to be protected; instead, sensing accuracy is characterized through estimation metrics such as the Fisher information matrix and the Cram\'er Rao lower bound, which are imposed as constraints to guarantee reliable angular estimation while enhancing communication secrecy.
}

Secure ISAC design in practice remains challenging due to idealized assumptions and high complexity. Many existing approaches rely on perfect or partially known CSI for both users and Eves, which is unrealistic in mobile, fading, or cluttered environments. Others depend on computationally heavy joint optimizations of waveform, beamforming, and power allocation, making them impractical for real-time or edge deployment. Moreover, most prior studies are confined to narrowband scenarios, offering limited insight for scalable multicarrier systems. Another critical limitation lies in the use of angular information. Secure beamforming often assumes highly accurate AoA estimates and perfect array response models, yet in practice, AoA estimation is sensitive to noise and dynamics, and real antenna arrays suffer from impairments and mismatches. These issues create a fundamental gap: existing methods are fragile under CSI uncertainty, angular errors, hardware imperfections, and the practical demands of multicarrier ISAC.

To address these challenges, we propose a deep learning–based secure ISAC framework that avoids reliance on Eve's CSI and explicitly accounts for AoA errors. Our design leverages ISAC's sensing capability: radar echoes are used to identify vulnerable spatial regions, guiding a radar-assisted jamming controller. An autoencoder learns robust transmission policies under fading, while a TT-Q compression scheme ensures scalability and efficiency. Unlike conventional optimization-based methods, the proposed approach adapts to complex channel dynamics, supports real-time multicarrier deployment, and jointly optimizes secrecy, reliability, and sensing utility.

In summary, the main contributions of this paper are summarised as follows:
\begin{itemize}
    \item We propose a learning-based secure ISAC framework that integrates radar sensing with communication optimization. Using real-time radar echoes, the system adaptively identifies vulnerable spatial regions and directs friendly jamming, without requiring Eve's CSI or location knowledge.

    \item We develop a nonparametric method to estimate the FIM from radar echo distributions via $f$-divergence, enabling CRLB-guided jamming under sensing uncertainty. Trained on perturbed radar and channel samples, it remains robust to angular errors, mobility, and hardware imperfections.
    
  \item We design the framework to support both overlapping and non-overlapping subcarrier allocations, allowing fine-grained jamming and secure operation across multicarrier ISAC systems, including OFDM-based implementations.
  
   \item We introduce a compact encoder based on TT-Q decomposition, which achieves a parameter reduction of more than two orders of magnitude. This lightweight design facilitates real-time inference with minimal memory overhead while maintaining robust sensing and jamming performance.

\end{itemize}

\subsection{Paper Organization and Notations}
The remainder of this paper is organized as follows. Section~\ref{sec:single_carrier} presents the system model and friendly jamming design for single-carrier ISAC, including CRLB-based constraints and the nonparametric Fisher information estimator. Section~\ref{sec:multicarrier} extends the framework to multicarrier systems, focusing on subcarrier-level optimization and robustness to imperfect CSI and angular errors. Section~\ref{sec:nonoverlap} studies non-overlapping subcarrier allocation to separate communication from sensing/jamming and reduce interference. Section~\ref{sec:Numerical_Results} reports simulation results on secrecy rate, BLER, convergence, and CRLB robustness. Section~\ref{sec:conclusion} concludes with future research directions.

\textit{Notation:} Vectors and matrices are denoted by bold lowercase and bold uppercase letters, respectively. The operators $(\cdot)^*$, $(\cdot)^\mathrm{T}$, and $(\cdot)^\mathrm{H}$ denote complex conjugate, transpose, and Hermitian, while $\mathrm{trace}(\cdot)$ denotes the trace of a matrix. The modulus of a complex number is written as $|\cdot|$, and the Euclidean norm of a vector as $|\cdot|$. The partial derivative of a vector $\mathbf{a}$ with respect to a variable $o$ is denoted $\dot{\mathbf{a}}_o$. The notation $\mathcal{CN}(\mu, \sigma^2)$ represents a circularly symmetric complex Gaussian distribution with mean $\mu$ and variance $\sigma^2$, and $\mathbb{E}{\cdot}$ denotes expectation.

\section{Friendly Jamming-enabled Secure Single-Carrier ISAC}
\label{sec:single_carrier}
{\color{black} 
In this work, the base station does not attempt to explicitly identify or classify eavesdroppers among surrounding devices. Instead, consistent with standard physical layer security assumptions, any unintended receiver outside the set of legitimate users is treated as a potential eavesdropper. Radar sensing is used to infer coarse spatial and angular information about the surrounding environment, such as dominant reflection directions and angular regions of interest, rather than to determine device identities. Friendly jamming is then designed to suppress information leakage toward these angular regions in a worst-case manner, without requiring explicit knowledge of the number, location, or identity of eavesdroppers.
}
\subsection{System Model}
We consider a MIMO ISAC system where a base station (BS), equipped with $N_t$ transmit and $N_r$ receive antennas, serves $K$ single-antenna users while simultaneously performing sensing and FJ against a passive Eve with $N_e$ antennas. The system operates over a single-carrier flat-fading channel, as illustrated in Fig.~\ref{fig:ISAC_FJ_System_Journal}. The channel vector from the BS to user $k$ is denoted by $\mathbf{h}_k \in \mathbb{C}^{N_t}$, and the channel matrix from the BS to Eve is $\mathbf{H}_E \in \mathbb{C}^{N_t \times N_e}$. 
The estimated aggregated downlink channel matrix to all legitimate users is given by
$\mathbf{H}_B \triangleq 
\begin{bmatrix}
\mathbf{h}_1^H & \mathbf{h}_2^H & \cdots & \mathbf{h}_K^H
\end{bmatrix}^H \in \mathbb{C}^{K \times N_t}$. 
In this work, we do not focus on the CSI acquisition process itself. Instead, we assume the BS obtains imperfect CSI through standard pilot-based estimation and model the estimated channel as $\hat{\mathbf{H}}_B = \mathbf{H}_B + \Delta \mathbf{H}_B$, where $\Delta \mathbf{H}_B$ captures estimation errors due to noise, quantitation, and mobility.

\begin{figure}[t]
\centering
\includegraphics[width=\linewidth]{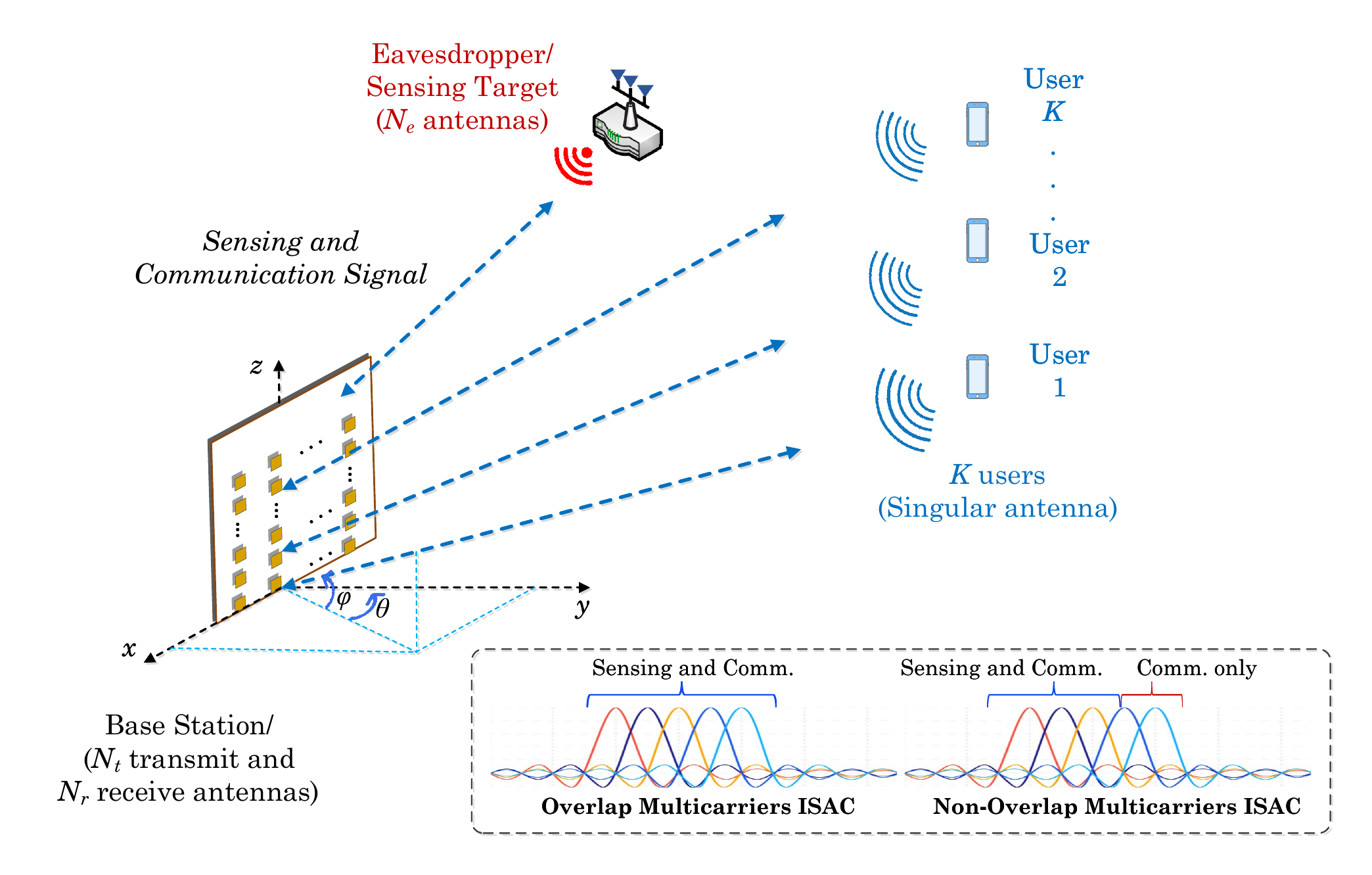}
\caption{{\color{black} Two multicarrier ISAC architectures considered in this work. 
In the overlapping scheme, all subcarriers are jointly used for communication, sensing, and friendly jamming. 
In the non-overlapping scheme, the subcarriers are partitioned into disjoint sets, with one set allocated to communication and the remaining set used for sensing and friendly jamming.}}
\label{fig:ISAC_FJ_System_Journal}
\end{figure}

The BS transmits a superposition of communication and FJ signals, which can be expressed as follows:
\begin{align}
\mathbf{x} = \sum\nolimits_{k=1}^{K} \mathbf{f}_k s_k + \mathbf{v}_\text{FJ} \eta
\end{align}
where $\mathbf{f}_k \in \mathbb{C}^{N_t}$ and $s_k$ are the beamformer and the data symbol intended for user $k$, respectively; $\mathbf{v}_\text{FJ}$ is the FJ beamforming vector and $\eta \sim \mathcal{CN}(0,u)$ is the jamming signal.
To prevent interference with legitimate users during sensing and jamming, the FJ beamforming vector $\mathbf{v}_\text{FJ}$ is restricted to lie in the null space of $\hat{\mathbf{H}}_B$, \textit{i.e.} $\mathbf{v}_\text{FJ} \in \mathcal{N}(\hat{\mathbf{H}}_B)$.

The received signals at user $k$ and Eve can be respectively expressed as
\begin{align}
y_k &= \mathbf{h}_k^H \mathbf{x} + n_k = \mathbf{h}_k^H \mathbf{f}_k s_k + \sum_{j \ne k} \mathbf{h}_k^H \mathbf{f}_j s_j + \mathbf{h}_k^H \mathbf{v}_\text{FJ} \eta + n_k, \\
\mathbf{y}_E &= \mathbf{H}_E^H \mathbf{x} + \mathbf{n}_E = \sum_{k=1}^{K} \mathbf{H}_E^H \mathbf{f}_k s_k + \mathbf{H}_E^H \mathbf{v}_\text{FJ} \eta + \mathbf{n}_E.
\end{align}
As a result, the achievable rate at user $k$ is
\begin{align}
R_{B,k} = \log_2\Big(1 + \frac{|\mathbf{h}_k^H \mathbf{f}_k|^2}{\sum_{j \ne k} |\mathbf{h}_k^H \mathbf{f}_j|^2 + |\mathbf{h}_k^H \mathbf{v}_\text{FJ}|^2 + \sigma_k^2} \Big)
\end{align}
where $n_k \sim \mathcal{CN}(0,\sigma_c^2)$ denotes the additive Gaussian noise at the $k$-th legitimate user, and 
$\mathbf{n}_E \sim \mathcal{CN}(\mathbf{0},\sigma_E^2 \mathbf{I})$ represents the additive Gaussian noise vector at  Eve. The information leakage rate at Eve due to user $K$'s signal is
{\small\begin{align}
\label{ep:Re_rate}
R_{E,k} = \log_2\det\Big( \mathbf{I} + \big(\zeta \mathbf{H}_E^H \mathbf{v}_\text{FJ} \mathbf{v}_\text{FJ}^H \mathbf{H}_E + \sigma_E^2 \mathbf{I} \big)^{-1} \mathbf{H}_E^H \mathbf{f}_k \mathbf{f}_k^H \mathbf{H}_E \Big).
\end{align}}

{\color{black}
Although Eve is assumed to be passive and does not transmit probing signals, it is assumed to have perfect knowledge of its downlink channel $\mathbf{H}_E$ for decoding purposes. The information leakage rate in \eqref{ep:Re_rate} is evaluated using a MIMO capacity formulation, which implicitly assumes optimal multi-antenna receive processing at the eavesdropper. Further, potential collusion among multiple passive eavesdroppers can be equivalently represented by an increased effective number of eavesdropper antennas $N_e$, which is naturally accommodated in the considered formulation.
}

\subsection{Problem Formulation of Single-carrier ISAC-FJ}
Unlike conventional approaches that exploit all available null-space dimensions for FJ transmission \cite{akgun2016exploiting, siyari2017friendly}, our method selectively activates FJ beams based on detected sensing targets. Specifically, the BS evaluates candidate beams $\mathbf{v}_i \in \mathcal{N}(\hat{\mathbf{H}}_B)$ using radar echo observations:
\begin{align}
\mathbf{y}_s^{(i)} = \alpha \mathbf{G}(\theta, \phi) \left( \sum\nolimits_{k=1}^K \mathbf{f}_k s_k + \mathbf{v}_i \eta \right) + \mathbf{n}_s
\end{align}
where $\alpha$ is the reflection coefficient, the angles $\theta$ and $\phi$ denote the azimuth and elevation angles of arrival, respectively. $\mathbf{G}(\theta, \phi)$ is the angular steering matrix, and $\mathbf{n}_s \sim \mathcal{CN}(0, \sigma_s^2 \mathbf{I})$. The steering matrix is given by $\mathbf{G}(\theta, \phi) = \mathbf{a}_x(\theta, \phi) \otimes \mathbf{a}_y(\theta, \phi)$,
where $\alpha$ is the reflection coefficient, $\mathbf{G}(\theta,\phi)$ is the angular steering matrix, and $\mathbf{n}_s \sim \mathcal{CN}(0,\sigma_s^2\mathbf{I})$. The angles $\theta$ and $\phi$ denote the azimuth and elevation, respectively. The steering matrix is given by
\begin{align}
\mathbf{G}(\theta,\phi) = \mathbf{a}_x(\theta,\phi) \otimes \mathbf{a}_y(\theta,\phi)
\end{align}
where $\mathbf{a}_x$ and $\mathbf{a}_y$ are the steering vectors along the $x$- and $y$-axes, $\theta$ and $\phi$ are the and $\otimes$ denotes the Kronecker product. 

To quantify angular estimation performance, we employ the Fisher information for the angular parameters 
$\omega \in \{\theta,\phi\}$. 
Since the FIM for the parameter vector $[\theta,\phi]^T$ is in general $2\times2$, we consider the scalar Fisher information corresponding to each parameter independently. 
The Fisher information for parameter $\omega$ is defined as
\begin{equation}
    J_{\omega} \triangleq 
    \frac{2\rho}{\sigma_s^2} 
    \bigg\| \frac{\partial \mathbf{G}(\theta,\phi)\mathbf{v}_{\mathrm{FJ}}}{\partial \omega} \bigg\|^2
\end{equation}
where $\rho$ is the SNR. 
The corresponding CRLB is given by
\begin{equation}
\mathrm{CRLB}_\omega = \frac{1}{J_{\omega}}
= \frac{\sigma_s^2}{2\rho \Big\| \frac{\partial \mathbf{G}(\theta,\phi)\mathbf{v}_{\mathrm{FJ}}}{\partial \omega} \Big\|^2}
\end{equation}
which gives the minimum variance of any unbiased estimator of $\omega$. Smaller CRLB values indicate higher accuracy in estimating the azimuth ($\theta$) and elevation ($\phi$) angles.

We adopt a worst-case scenario to ensure secure communication in the presence of uncertain CSI of Eve. Specifically, Eve is assumed to decode under its most favourable channel realization. The BS then seeks to maximize the users’ sum secrecy rate while enforcing constraints on $\mathrm{CRLB}_{\theta}$ and $\mathrm{CRLB}_{\phi}$. The secrecy rate optimization problem is formulated as:
\begin{subequations} \label{eq:problem10}
	\begin{IEEEeqnarray}{cl}
		\ &\underset{\{\gamma_k\},\,\{\zeta\}}{\mathrm{max}} \quad   \sum_{k=1}^{K} \left[ R_{B,k} - \max(R_{E,k}) \right]^+
		\label{eq:single_eq} \\
		& \st  \ \mathrm{CRLB}_{\omega}\!\left(\mathbf{v}_{\mathrm{FJ}}, \zeta\right)
   \leq \mathrm{CRLB}_{0,\omega}, \ \forall \omega \in \{\theta,\phi\}
\label{eq:single_eq_a} \\
&\quad\ \mathbf{v}_{\mathrm{FJ}} \in \mathcal{N}\!\left(\hat{\mathbf{H}}_{B}\right)
\label{eq:single_eq_b} \\
&\quad\ \sum_{k=1}^{K} \|\mathbf{f}_k\|^2 + \zeta \|\mathbf{v}_{\mathrm{FJ}}\|^2
   \leq P_{\max}
\label{eq:single_eq__c}
	\end{IEEEeqnarray}
\end{subequations}
where $[x]^+ = \max\{0,x\}$ and $P_{\max}$ is the total transmit power budget at the BS. CRLB constraints in \eqref{eq:single_eq_a} ensure that the probing and FJ design at the BS maintains a guaranteed level of sensing accuracy. 
Since the CRLB represents the fundamental lower bound on the variance of any unbiased AoA estimator, imposing 
$\mathrm{CRLB}_\omega(\mathbf{v}_{FJ},\rho) \leq \mathrm{CRLB}_{0,\omega}$ guarantees that the BS can still obtain reliable angular estimates even in the presence of friendly jamming. This makes the CRLB a tractable and meaningful metric for joint optimization: it is deterministic, depends only on the transmit design variables, and allows the BS to explicitly trade off secrecy enhancement against sensing performance.


\begin{figure*}[!htb]
\centering
\includegraphics[width=.9\linewidth]{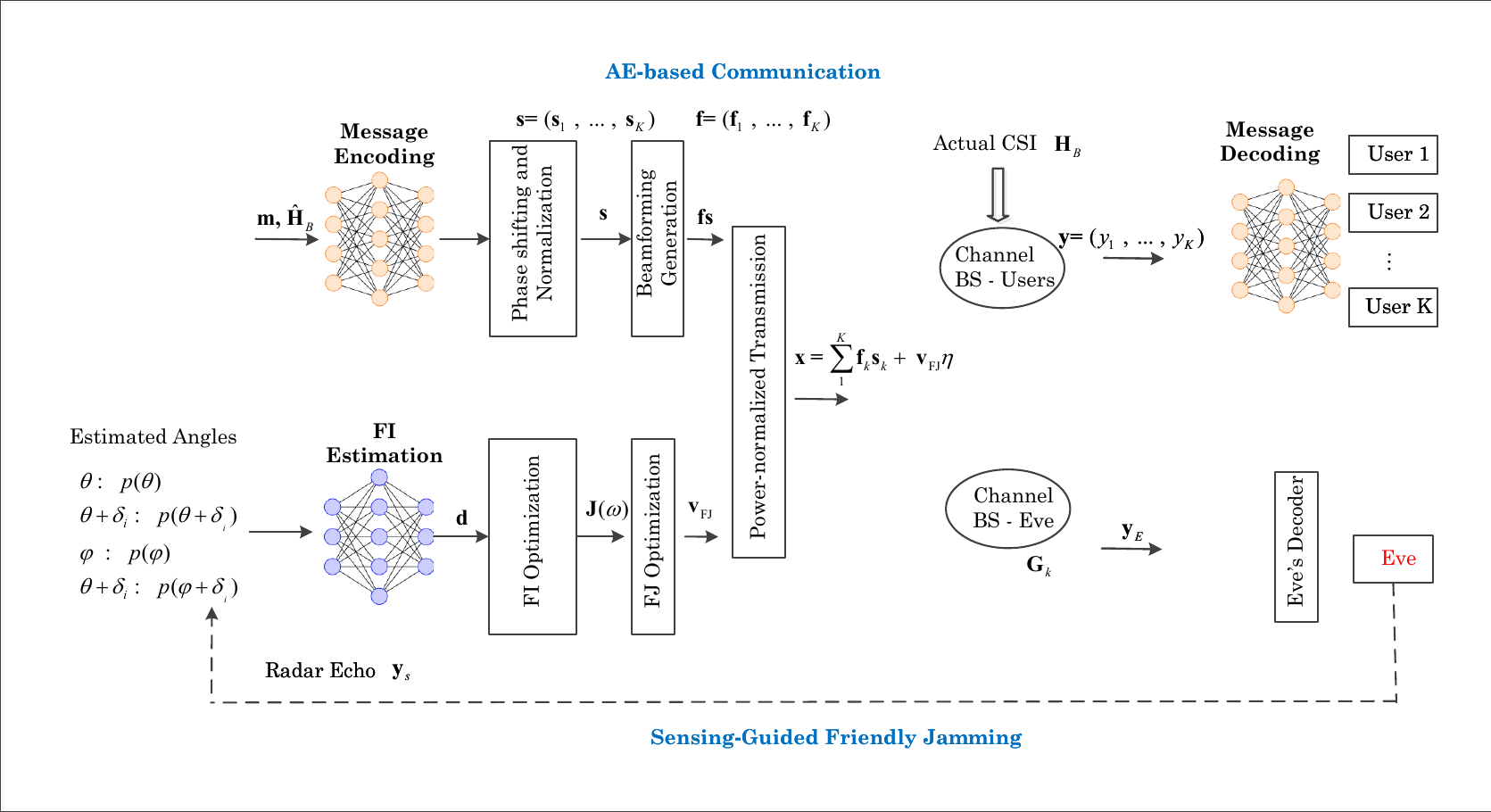}
\caption{{\color{black} ISAC FJ workflow: (1) message encoding, (2) beamforming generation, (3) power-normalized transmission, (4) radar echo acquisition, (5) Fisher information estimation, (6) friendly jamming optimization, and (7) message decoding. The dashed arrow highlights the sensing feedback loop, which enables adaptive jamming without requiring explicit eavesdropper CSI. }}
\label{fig:proposed_framework}
\end{figure*}

The proposed framework, illustrated in Fig.~2, comprises two components: (i) a communication learning module and (ii) a sensing-guided FJ mechanism. Instead of relying on closed-form FIM derivations under Gaussian assumptions, we adopt a data-driven nonparametric approach based on $f$-divergence \cite{berisha2014empirical, berisha2015empirically}, which estimates the FIM directly from radar echo samples without requiring explicit likelihood models.

The received signal has mean $\mathbf{x}_s(\omega, \omega)$ and covariance $\sigma_s^2 \mathbf{I}$, the FIM with respect to parameter vector $\omega$ is:
\begin{align}
\mathbf{J}_\omega = \frac{2}{\sigma_s^2} \, \mathfrak{R}\left( \frac{\partial \mathbf{x}_s^H}{\partial \omega} \frac{\partial \mathbf{x}_s}{\partial \omega} \right).
\end{align}
The communication module at the BS learns to map user messages to beamformers via an end-to-end model. Each input message $m_k \in \mathcal{M}$ is represented as a one-hot vector of size $2 \times M$, with $M$ denoting the number of codewords.  It then generates beamformers $\mathbf{f}_k$ while the decoder reconstructs messages from received signals. Training minimizes the categorical cross-entropy loss and an additional unsupervised objective that maximizes the total achievable rate $R_s = \sum_k R_{B,k}$ through the rate loss $\mathcal{L}_{\text{rate}} = -R_s$. The generation of FJ beams satisfying the CRLB constraints is detailed in the following section.

\subsection{FJ-based Fisher Information Estimation}
The effectiveness of FJ depends on angular estimation accuracy, quantified by the CRLB. Assuming target angles $(\theta, \phi)$ are estimated (\textit{e.g.}, via MUSIC), we adopt the $f$-divergence-based method in \cite{berisha2014empirical} to infer the FIM.

For each estimated channel realization $\mathbf{H}_B$, we generate perturbed angular samples $(\theta + \delta_{i1}, \phi + \delta_{i2})$, where $\delta_i \in \mathbb{R}^d$ represents small perturbations and $d$ is the parameter dimension. Since both azimuth and elevation are estimated jointly, $d = 2$. To recover the symmetric $d \times d$ FIM, at least $D \geq d(d+1)/2$ perturbed samples are required. The divergence between the original and perturbed sample distributions is then approximated as:
\begin{align}
D_f(p(\theta, \phi) \,\|\, p((\theta+\delta_{i1}, \phi+\delta_{i2})) \approx \frac{1}{2} \delta_i^\top \mathbf{J}_\theta^{(i)} \delta_i.
\end{align}
{\color{black}
The proposed $f$-divergence-based nonparametric FIM estimator operates in a local perturbation regime, where the divergence between the nominal and perturbed radar echo distributions admits a second-order quadratic approximation, as shown in (12). By explicitly controlling the perturbation magnitude and the number of perturbed samples during training, the estimator remains within this local regime, ensuring numerical stability of the least-squares FIM reconstruction in (15). Importantly, the reconstructed FIM is defined with respect to the angular parameter vector (azimuth and elevation) and therefore has a fixed dimension of $2\times2$, independent of the number of antennas or subcarriers. As a result, increasing the antenna array size does not increase the dimensionality of the FIM estimation problem and does not lead to high-dimensional instability. In addition, the FIM and the corresponding CRLB values are evaluated using batch-averaged radar echo samples rather than a single realization, which reduces stochastic fluctuations. The positive semidefinite constraint imposed in (16c) further guarantees physically meaningful FIM estimates and stabilizes the resulting CRLB constraints used for friendly jamming optimization.
}

Regarding Fisher information neural estimation, which is first investigated in \cite{duy2022fisher}, the neural discriminator is denoted by $T_i$ associated with perturbation $\delta_i$. 
It is trained to approximate the $f$-divergence between the original distribution 
$p(\theta,\phi)$ and its perturbed counterpart $p((\theta+\delta_{i1}, \phi+\delta_{i2}))$. 
For each sample $\omega_{i,j}$, the output $T_i(\omega_{i,j})$ provides a scalar score 
that is used to estimate the divergence in \eqref{eq:divergent}. The neural discriminator $T_i$ estimates each divergence as:
\begin{align}
d(\delta_i) \approx \frac{1}{N} \sum_{j=1}^N T_i(\mathbf{\omega}_{i, j}) - \log\Big( \frac{1}{N} \sum_{j=1}^N e^{T_i(\mathbf{\omega}_{i,j})} \Big)
\label{eq:divergent}
\end{align}
where $N$ is the number of samples. 
The resulting divergence vector $\mathbf{d} \triangleq [d(\delta_1), \dots, d(\delta_M)]^\top$ relates to the vectorized FIM via: $2 \mathbf{d} = \mathbf{U} \mathbf{f}$,
where each row $\mathbf{u}_i$ is constructed as:
\begin{align}
\mathbf{u}_i = [\delta_{i1}^2, \dots, \delta_{id}^2, 2\delta_{i1}\delta_{i2}, \dots, 2\delta_{i(d-1)}\delta_{id}]^\top.
\end{align}
Finally, the least-squares estimate of $\mathbf{f}$ is given by:
\begin{align}
\hat{\mathbf{f}}_{\text{LS}} = 2 (\mathbf{U}^\top \mathbf{U})^{-1} \mathbf{U}^\top \mathbf{d}.
\end{align}
This yields the optimal jamming direction $\mathbf{v}_\text{FJ}$ by aligning with the beam that maximizes angular sensitivity:
\begin{subequations}\label{eq:P11a}
\begin{IEEEeqnarray}{cl}
\min_{\mathbf{f}} \quad & \| 2\mathbf{d} - \mathbf{U} \mathbf{f} \|_2^2 \\
\text{s.t.} \quad & f_k = \hat{f}_k^{\text{LS}}, \quad \forall k \in \{1,\cdots,d\} \label{eq:15b}\\
& \mathrm{mat}(\mathbf{f}) = \mathbf{J}_\theta \succeq 0.
\end{IEEEeqnarray}
\end{subequations}
Problem \eqref{eq:P11a} ensures consistency with the least-squares estimate while constraining the reconstructed FIM to be positive semidefinite. This guarantees a physically meaningful solution that captures the system’s angular sensitivity, enabling the friendly jamming beam $\mathbf{v}_{\text{FJ}}$ to be steered toward the directions most susceptible to eavesdropping.

\subsection{Learning-based FJ for ISAC Framework}
The proposed ISAC-based FJ framework, illustrated in Fig.~2, enables the BS to train an end-to-end communication model that jointly optimizes the beamforming vectors for legitimate users while accounting for imperfect CSI. Each input message $\textbf{m}$ is mapped to a beamforming vector, conditioned on the estimated CSI, which is then transmitted and decoded at the receiver. To enhance PLS under uncertainty, the BS concurrently probes multiple candidate beams from the null space $\mathcal{N}(\hat{\mathbf{H}}_B)$ and collects the associated radar echo samples. These samples are used to optimize the FJ beam by maximizing angular sensitivity.

The overall training objective integrates several components: communication performance, FIM estimation (FIME), positive semidefinite (PSD) enforcement on the FIM, and CRLB satisfaction. The FJ optimization problem is formulated as:
\begin{subequations} 
	\begin{IEEEeqnarray}{cl}
\min_{\mathbf{v}_{\mathrm{FJ}}} \quad & \sum_{\omega \in \{\theta, \phi\}} \text{CRLB}_\omega (\mathbf{v}_\text{FJ}, u)  \\
\text{s.t.} \quad & \text{CRLB}_\omega(\mathbf{v}_\text{FJ}) \leq \text{CRLB}_{0,\omega}, \quad \forall \omega \in \{\theta, \phi\}.
\end{IEEEeqnarray}\end{subequations} 
This formulation ensures that the optimized FJ beam: $i$) maximizes Fisher information for angular parameters, $ii$) lies within the null space of the legitimate channel to prevent communication interference, $iii$) results in a valid PSD FIM, and $iv$) satisfies predefined CRLB-based estimation accuracy constraints.

Following training, the BS selects the FJ beam $\mathbf{v}_\text{FJ}$ that maximizes the trace of the estimated FIM, effectively steering energy toward the most angularly sensitive directions in the environment. The FIME loss is defined as: $\mathcal{L}_{\mathrm{FIME}} = \mathrm{CRLB}_\theta + \mathrm{CRLB}_\phi$, which captures the overall estimation uncertainty for angular parameters. The training process is outlined in Algorithm~\ref{alg:sequential_training}.

\begin{algorithm}[t]
\caption{Neural ISAC-Based Friendly Jamming Optimization}
\label{alg:sequential_training}
\begin{algorithmic}[1]
\small
\Require Estimated channel $\hat{\mathbf{H}}_B$, initialized encoder/decoder, initial FJ beam $\mathbf{v}_{\mathrm{FJ}}$.
\State \textbf{Stage 1: Communication Beamforming Optimization}
\For {each communication training epoch}
    \State Sample a mini-batch of input messages $\{m_k\}$;
    \State Generate communication beams $\{\mathbf{f}_k\}$;
    \State Compute communication loss $\mathcal{L}_{\text{rate}}$;
    \State Update encoder and decoder parameters using gradients of $\mathcal{L}_{\text{rate}}$;
\EndFor

\State \textbf{Stage 2: Friendly Jamming Beam Optimization via FIME}
\For {each FIME training iteration}
    \State Fix optimized communication beams $\{\mathbf{f}_k\}$;
    \State Probe candidate FJ beams $\mathbf{v}_{\mathrm{FJ}}$ in the null space of $\hat{\mathbf{H}}_B$;
    \State Estimate FIM using radar echoes;
    \State Compute CRLB values for angles $\theta$ and $\phi$;
    \If {CRLB$_\theta \leq$ CRLB$_{0,\theta}$ \textbf{and} CRLB$_\phi \leq$ CRLB$_{0,\phi}$}
        \State Accept update: minimize $\mathcal{L}_{\text{FIME}}$ and update $\mathbf{v}_{\mathrm{FJ}}$;
    \Else
        \State Reject update: discard gradient step and optionally reinitialize $\mathbf{v}_{\mathrm{FJ}}$;
    \EndIf
\EndFor
\Return The optimal beamformers: $\{\mathbf{f}_k^*\}_{\forall k}$ and $\mathbf{v}^*_{\mathrm{FJ}}$.
\end{algorithmic}
\end{algorithm}

\section{Friendly Jamming Secure  Multicarrier ISAC}
\label{sec:multicarrier}
Single-carrier ISAC systems face limitations in handling frequency-selective fading, spatial resolution, and joint sensing-communication optimization, especially under practical impairments like imperfect CSI and dynamic environments \cite{xiao2024novel, nguyen2023multiuser, liu2022integrated}. In contrast, multicarrier OFDM-based ISAC offers greater spectral efficiency, finer frequency resolution, and more design flexibility. Subcarrier-level adaptation enables precise beamforming, robust power allocation, and frequency-selective jamming, making OFDM more practical for secure ISAC \cite{cui2021integrating}. Further, unlike the perfect AoA estimation discussed in Section \ref{sec:single_carrier}, the imperfect AoA will be considered in the worst-case scenario.
The system model is illustrated in Fig.~\ref{fig:ISAC_FJ_System_Journal}, 
which shows both overlapping and non-overlapping multicarrier ISAC configurations. 
In this section, we will present the overlap multicarrier scheme, 
where the entire sub-bandwidth is simultaneously exploited for both sensing and communication.

\subsection{Multicarrier MIMO-OFDM System Model}
Before diving into the multicarrier system model, we highlight a key property of our design: the sensing accuracy, measured by the FIM and the resulting CRLB, depends only on the transmit covariance and not on the detailed pilot waveform. This guarantees that our framework is valid for a wide class of OFDM implementations and subcarrier allocations, as formalized below.

\begin{theorem}[Waveform-agnostic FIM for OFDM probing]\label{thm:ofdm_agnostic}
Consider the received radar echo on subcarrier $n \in \mathcal{S}_r$ modeled as
\[
\mathbf{Y}_r[n] \sim \mathcal{CN}\!\big(\boldsymbol{\mu}[n],\,\sigma_r^2\mathbf I\big), 
\qquad 
\boldsymbol{\mu}[n] = \mathbf G(\boldsymbol{\vartheta})\,\mathbf S_r[n]
\]
where $\boldsymbol{\vartheta}$ collects the angular parameters (\textit{e.g.}, AoA), 
$\mathbf G(\boldsymbol{\vartheta})$ is the array/system response, 
and $\mathbf S_r[n]\in\mathbb C^{N_t\times T}$ is the OFDM probing matrix 
with $N_t$ antennas and $T$ time slots. By defining 
$\mathbf A_\omega := \tfrac{\partial \mathbf G}{\partial \omega}$ 
and $\mathbf R_s[n] := \mathbf S_r[n]\mathbf S_r[n]^{\!H}$, the Fisher information for parameter $\omega$ is
\begin{equation}\label{eq:FIM_ofdm}
J_\omega \;=\; \frac{2}{\sigma_r^2}\sum_{n\in\mathcal{S}_r} 
\operatorname{tr}\!\Big(\mathbf R_s[n]\,\mathbf A_\omega^{\!H}\mathbf A_\omega\Big).\end{equation}
It is clear that $J_\omega$ depends on $\mathbf S_r[n]$ only through the transmit covariance $\mathbf R_s[n]$. 
In particular, for any unitary $\mathbf U\in\mathbb C^{T\times T}$, 
if $\mathbf S'_r[n]=\mathbf S_r[n]\mathbf U$, then $\mathbf R'_s[n]=\mathbf R_s[n]$ 
and hence $J'_\omega=J_\omega$.
\end{theorem}

\begin{proof}[Proof sketch]
For Gaussian white noise with the mean depending on $\omega$, the FIM per subcarrier is
\begin{equation}
J_\omega(n) = \frac{2}{\sigma_r^2}\,\big\|\partial_\omega \boldsymbol{\mu}[n]\big\|_F^2
= \frac{2}{\sigma_r^2}\operatorname{tr}\!\Big((\partial_\omega \boldsymbol{\mu}[n])^{\!H}
\,\partial_\omega \boldsymbol{\mu}[n]\Big).
\end{equation}
Since $\boldsymbol{\mu}[n]=\mathbf G(\boldsymbol{\vartheta})\,\mathbf S_r[n]$, 
we obtain $\partial_\omega \boldsymbol{\mu}[n]=\mathbf A_\omega \mathbf S_r[n]$. 
From the cyclic property of trace, it follows that
\begin{align}
J_\omega(n) 
  &= \frac{2}{\sigma_r^2}\operatorname{tr}\!\big(
     \mathbf S_r[n]^{\!H}\mathbf A_\omega^{\!H}\mathbf A_\omega \mathbf S_r[n]\big) \\
  &= \frac{2}{\sigma_r^2}\operatorname{tr}\!\big(
     \underbrace{\mathbf S_r[n]\mathbf S_r[n]^{\!H}}_{\mathbf R_s[n]}
     \mathbf A_\omega^{\!H}\mathbf A_\omega\big).
\end{align}
 Summing over $n\in\mathcal{S}_r$ gives \eqref{eq:FIM_ofdm}. 
If $\mathbf S'_r[n]=\mathbf S_r[n]\mathbf U$ with $\mathbf U$ unitary, then 
$\mathbf R'_s[n]=\mathbf R_s[n]$ and hence $J'_\omega=J_\omega$.
\end{proof}

\noindent
\textbf{Implication:} Any OFDM probing waveform that yields the same covariance 
$\{\mathbf R_s[n]\}$ (\textit{i.e.} same per-subcarrier power/aperture distribution) 
produces identical $J_\omega$ and CRLB. 
Thus, the different pilot placements or symbol mappings but identical $\mathbf R_s[n]$ are equivalent in terms of Fisher information. 
With colored noise covariance $\boldsymbol{\Sigma}$, the result generalizes to
\begin{align}
J_\omega = 2\sum_{k\in\mathcal{S}_r}\operatorname{tr}\!\big(\mathbf R_s[n]\,
\mathbf A_\omega^{\!H}\boldsymbol{\Sigma}^{-1}\mathbf A_\omega\big)
\end{align}
which still depends on $\mathbf S_r[n]$ only through $\mathbf R_s[n]$. For wideband or beam-squint arrays, replacing $\mathbf{A}_\omega$ by its frequency-dependent form $\mathbf{A}_\omega[n]$ preserves the same conclusion on a per-subcarrier basis.

\subsection{Problem Formulation of Overlap Multicarrier ISAC-FJ}
We consider a multicarrier MIMO-OFDM ISAC system operating over $N$ orthogonal subcarriers indexed by $n \in \{1, \ldots, N\}$, while also performing radar sensing and FJ. Each user $k$ transmits a data symbol $s_k^{(n)} \in \mathbb{C}$ on subcarrier $n$, with $\mathbb{E}[|s_k^{(n)}|^2] = 1$. 

The transmitted signal consists of two distinct components: 
(i) structured OFDM pilots, which serve both communication channel estimation and radar sensing, 
and (ii) a friendly jamming (FJ) signal, which is designed to degrade Eve’s reception. 
Only the OFDM pilots contribute to radar echo processing and parameter estimation, 
while the FJ signal acts solely as interference. 
The transmitted signal from the BS on subcarrier $n$ can be expressed as
\begin{align}
\mathbf{x}^{(n)} = \sum_{k=1}^K \sqrt{\xi_k^{(n)}} \mathbf{f}_k^{(n)} s_k^{(n)} + \sqrt{\zeta^{(n)}} \mathbf{v}^{(n)} \eta^{(n)}
\end{align}
where $\xi_k^{(n)} $ is the communication power allocation coefficient, $\zeta^{(n)}$ is the jamming power coefficient, and $\eta^{(n)} \sim \mathcal{CN}(0,1)$ is the jamming signal. By OFDM, the received signal at user $k$ on subcarrier $n$ is given by
\begin{align}
y_k^{(n)} = (\mathbf{h}_k^{(n)})^H \mathbf{x}^{(n)} + n_k^{(n)}.
\end{align}
The radar echo received at the BS across all subcarriers can be compactly expressed as
\begin{align}
    \mathbf{Y}_s = \mathbf{G}(\theta, \phi) \, \boldsymbol{\alpha} \, \mathbf{X} + \mathbf{N}_s
\end{align}
where \( \mathbf{X} = [\mathbf{x}^{(1)}, \dots, \mathbf{x}^{(N)}] \in \mathbb{C}^{N_t \times N} \) is the transmit signal matrix across \( N \) subcarriers, \( \mathbf{Y}_s \in \mathbb{C}^{N_r \times N} \) is the corresponding stacked radar echo, and \( \mathbf{N}_s \in \mathbb{C}^{N_r \times N} \) is the radar noise matrix. The bistatic sensing channel \( \mathbf{G}(\theta, \phi) = \mathbf{b}(\theta, \phi)\mathbf{a}^H(\theta, \phi) \in \mathbb{C}^{N_r \times N_t} \) is assumed constant across subcarriers, and the reflection coefficient is modeled as a diagonal matrix \( \boldsymbol{\alpha} = \mathrm{diag}(\alpha^{(1)}, \dots, \alpha^{(N)}) \in \mathbb{C}^{N \times N} \), where each \( \alpha^{(n)} \in \mathbb{C} \) captures the frequency-dependent reflectivity at subcarrier \( n \). Under the practical assumption that the target reflection is frequency-flat, we simplify with \( \alpha^{(n)} = \alpha \) for all \( n \), yielding \( \boldsymbol{\alpha} = \alpha \mathbf{I}_N \), and the model reduces to
\begin{align}
    \mathbf{Y}_s = \alpha \, \mathbf{G}(\theta, \phi) \, \mathbf{X} + \mathbf{N}_s.
\end{align}
The total transmit power averaged over all subcarriers is then constrained as
\begin{align}
\frac{1}{N} \sum_{n=1}^{N} \mathbb{E}\left[\| \mathbf{x}^{(n)} \|^2 \right]
&= \frac{1}{N} \sum_{n=1}^{N} \Big( 
\sum_{k=1}^{K} \xi_k^{(n)} \| \mathbf{f}_k^{(n)} \|^2 
+ \zeta^{(n)} \| \mathbf{v}^{(n)} \|^2 \Big 
) \notag \\
&\leq P_{\max}.
\end{align}

Let $\mathbf{H}_E^{(n)} \in \mathbb{C}^{N_t \times N_e}$ denote the downlink channel from the BS to Eve on subcarrier $n$. The achievable rate for user $k$ on subcarrier $n$, along with the corresponding information leakage rate to Eve, is defined in \eqref{eq:Rate_Bob_Eve_Norm}. Here, $\bar{\tau}$ represents the fraction of time or frequency resources allocated to communication. These expressions explicitly account for multiuser interference and the impact of jamming leakage at the receivers.

\begin{figure*}[!t]
\begin{equation}
\begin{split}
R_k^{(n)} &= \bar{\tau} \log_2 \Bigg( 
1 + \frac{ 
\xi_k^{(n)} \left| (\mathbf{h}_k^{(n)})^H \mathbf{f}_k^{(n)} \right|^2 
}{
\sum\limits_{j \ne k} \xi_j^{(n)} \left| (\mathbf{h}_k^{(n)})^H \mathbf{f}_j^{(n)} \right|^2 
+ \zeta^{(n)} \left| (\mathbf{h}_k^{(n)})^H \mathbf{v}^{(n)} \right|^2 + \sigma_c^2 
} \Bigg), \\
R_{E,k}^{(n)} &=\bar{\tau} \log_2 \det \Big( \mathbf{I} + 
\left( 
\zeta^{(n)} \mathbf{H}_E^{(n)H} \mathbf{v}^{(n)} \mathbf{v}^{(n)H} \mathbf{H}_E^{(n)} 
+ \sigma_E^2 \mathbf{I} 
\right)^{-1} 
\mathbf{H}_E^{(n)H} \mathbf{f}_k^{(n)} \mathbf{f}_k^{(n)H} \mathbf{H}_E^{(n)} 
\Big).
\end{split}
\label{eq:Rate_Bob_Eve_Norm}
\end{equation}
\hrulefill
\vspace*{0pt}
\end{figure*}

We denote by \( \gamma_k^{(n)} \) the communication power allocated to user \(k\) on subcarrier \(n\), and by \( \zeta^{(n)} \) the power allocated to FJ on subcarrier \(n\). We aim to jointly maximize the total secrecy rate while ensuring high-quality sensing and satisfying a global power constraint across all subcarriers. Given the multicarrier MIMO-OFDM model and the definitions of user rate \( R_k^{(n)} \) and leakage rate \( R_{E,k}^{(n)} \), the secrecy rate per user per subcarrier is defined as \( R_{\mathrm{sec},k}^{(n)} = \left[ R_k^{(n)} - R_{E,k}^{(n)} \right]^+ \). 

The resulting multicarrier secrecy-optimized power allocation problem can be formulated as follows:
\begin{subequations} \label{eq:problem}
	\begin{IEEEeqnarray}{cl}
		\ &\underset{\{\gamma_k^{(n)}\},\,\{\zeta^{(n)}\}}{\mathrm{max}} \quad   \sum_{n=1}^{N} \sum_{k=1}^{K}
\left[ R_k^{(n)} - R_{E,k}^{(n)} \right]^+
		\label{eq:objective29} \\
		& \st\  \frac{1}{N} \sum_{n=1}^{N}\Bigg(\sum_{k=1}^{K}
   \gamma_k^{(n)}\|\mathbf f_k^{(n)}\|^2
   + \zeta^{(n)}\|\mathbf v^{(n)}\|^2\Bigg) \leq P_{\max}\quad
\label{eq:constraint_a} \\[1ex]
&\,\quad \sum_{n=1}^{N} \mathrm{CRLB}_\theta^{(n)}
  \leq \mathrm{CRLB}_\theta^{\max}
\label{eq:constraint_b} \\
&\,\quad \sum_{n=1}^{N} \mathrm{CRLB}_\phi^{(n)}
  \leq \mathrm{CRLB}_\phi^{\max}.
\label{eq:constraint_c} 
	\end{IEEEeqnarray}
\end{subequations}
Constraint \eqref{eq:constraint_a} imposes a limit on the average total transmit power across all subcarriers, accounting for both the power allocated to the communication beams \( \mathbf{f}_k^{(n)} \) and the friendly jamming beam \( \mathbf{v}^{(n)} \). Constraints \eqref{eq:constraint_b} and \eqref{eq:constraint_c} ensure the accuracy of angle estimation for azimuth (\( \theta \)) and elevation (\( \phi \)), respectively, using the CRLB computed from Fisher information. If necessary, constraint \eqref{eq:constraint_a} can be reformulated as a per-subcarrier power constraint to enable more precise control of the transmit power allocation as follows:
\begin{equation}
\label{eq:constraint_26a_prime}
\sum_{k=1}^{K} \gamma_k^{(n)} \| \mathbf{f}_k^{(n)} \|^2 
+ \zeta^{(n)} \| \mathbf{v}^{(n)} \|^2 \leq \frac{P_{\max}}{N}, \quad \forall n.
\end{equation}

The transmit precoding matrix on each subcarrier consists of the stacked communication beams and the friendly jamming signal. Specifically, the total precoder can be written as \( \mathbf{F}^{(n)} = \mathbf{W}^{(n)} \boldsymbol{\Gamma}^{1/2} + \mathbf{v}^{(n)} \overline{\boldsymbol{\eta}}^T \), where \( \mathbf{W}^{(n)} = [\mathbf{f}_1^{(n)}, \ldots, \mathbf{f}_K^{(n)}] \in \mathbb{C}^{N_t \times K} \) is the communication beamformer matrix, and \( \overline{\boldsymbol{\eta}} \in \mathbb{C}^{1 \times N} \) is the time-domain sensing signal. To avoid interfering with legitimate users, the FJ beam \( \mathbf{v}^{(n)} \) is projected into the null space of the estimated channel matrix \( \hat{\mathbf{H}}_B^{(n)} = [\hat{\mathbf{h}}_1^{(n)}, \ldots, \hat{\mathbf{h}}_K^{(n)}] \), such that the following condition holds:
\begin{equation}
\mathbf{v}^{(n)} \in \mathcal{N}(\hat{\mathbf{H}}_B^{(n)}), \quad \forall n.
\end{equation}
The sensing quality on subcarrier \( n \) is quantified via the Fisher information is presented as follows:
\begin{equation}
\mathbf{J}_\omega^{(n)} = \frac{2 \zeta^{(n)}}{\sigma_s^2} \left\| \frac{\partial \mathbf{G}(\theta, \phi)\mathbf{v}^{(n)}}{\partial \omega} \right\|^2, \quad \omega \in \{\theta, \phi\}
\end{equation}
where the corresponding CRLB is given by:
\[
\mathrm{CRLB}_\omega^{(n)} = \frac{1}{\mathbf{J}_\omega^{(n)}}, \quad \omega \in \{\theta, \phi\}.
\]

Constraints \eqref{eq:constraint_b} and \eqref{eq:constraint_c} enforce per-subcarrier CRLBs to guarantee radar angle estimation accuracy across azimuth and elevation. The secrecy rate maximization in \eqref{eq:constraint_26a_prime} is challenging due to the non-convexity of both the objective and several constraints. The objective involves a difference of logarithmic terms for the legitimate and eavesdropper rates, both depending on interference and power allocation, yielding a non-convex formulation. Moreover, the CRLB constraints are nonlinear with matrix derivatives and norms, making analytical solutions intractable. 
Although constraint \eqref{eq:constraint_a} can be convexified, the coupling between communication powers \( \gamma_k^{(n)} \) and sensing/jamming power \( \zeta^{(n)} \) across subcarriers leads to a high-dimensional, non-separable problem. Techniques such as SCA, block coordinate descent, and difference-of-convex programming~\cite{liu2022integrated,nguyen2025performance,su2023sensing} provide locally optimal solutions but suffer from initialization sensitivity, slow convergence, and lack of global guarantees. Heuristic relaxations may also violate CRLB accuracy requirements and scale poorly with the number of users and subcarriers, limiting practical applicability. To address these challenges, a learning-based approach is proposed, where all constraints are embedded into the network and solved simultaneously.

\subsection{Imperfect Angle Estimation Model}
In practice, AoA estimation is constrained by noise, clutter, hardware nonlinearities, and finite antenna apertures, causing the estimated angles $(\hat{\theta}, \hat{\phi})$ to deviate from the true values $(\theta, \phi)$. We model these as noisy measurements: $\hat{\theta} = \theta + \delta_{\theta}$ and $\hat{\phi} = \phi + \delta_{\phi}$, where $\delta_{\theta} \sim \mathcal{N}(0, \sigma^2_{\theta})$ and $\delta_{\phi} \sim \mathcal{N}(0, \sigma^2_{\phi})$. The error variances depend on factors such as SNR, radar bandwidth, and antenna array size. These errors propagate through the steering vector $\mathbf{a}(\hat{\theta}, \hat{\phi})$, altering the sensing channel $\mathbf{G}(\hat{\theta}, \hat{\phi}) = \mathbf{b}(\hat{\theta}, \hat{\phi})\mathbf{a}^H(\hat{\theta}, \hat{\phi})$. Consequently, the FIM and CRLB may no longer reflect the true sensing geometry, leading to suboptimal estimation and degraded friendly jamming performance.

To address this issue, our framework explicitly accounts for angular uncertainty during training. Perturbed angle samples are drawn from the modelled distributions to generate diverse radar echo responses, which are then used as input to estimate the FIM via nonparametric $f$-divergence.

\subsection{Data-Driven Multicarrier FIM Estimation}
To address the challenges above, we adopt a nonparametric approach to estimate the FIM per subcarrier using $f$-divergence approximations. For each candidate jamming beam \( \mathbf{v}_i^{(n)} \in \mathcal{N}(\hat{\mathbf{H}}_B^{(n)}) \), radar echoes are collected for perturbed angles:
\begin{equation}
(\theta + \delta_{i1}^{(n)}, \phi + \delta_{i2}^{(n)}), \quad \forall n
\end{equation}
and a neural discriminator estimates divergence values \( d^{(n)}(\delta_i) \). Each FIM \( \mathbf{J}^{(n)} \) is reconstructed using:
\[
2 \mathbf{d}^{(n)} = \mathbf{U}^{(n)} \mathbf{f}^{(n)},
\quad
\hat{\mathbf{f}}^{(n)}_{\text{LS}} = 2 (\mathbf{U}^{(n)\top} \mathbf{U}^{(n)})^{-1} \mathbf{U}^{(n)\top} \mathbf{d}^{(n)}.
\]

To ensure physical validity, each estimated \( \mathbf{J}_\theta^{(n)} = \mathrm{mat}(\mathbf{f}^{(n)}) \) is constrained to be PSD:
\begin{subequations} \label{eq:problem34}
	\begin{IEEEeqnarray}{cl}
		\ &\underset{\mathbf{f}^{(n)}}{\mathrm{max}} \quad   \left\| 2 \mathbf{d}^{(n)} - \mathbf{U}^{(n)} \mathbf{f}^{(n)} \right\|_2^2
		\label{eq:objective} \\
		& \st\  f_k^{(n)} = \hat{f}_k^{(n),\text{LS}}, \quad \forall k
\label{eq:constraint_31a} \\[0.6ex]
& \quad\,\, \mathrm{mat}\!\left(\mathbf{f}^{(n)}\right) = \mathbf{J}_\theta^{(n)} \succeq 0.
\label{eq:constraint_31b}
	\end{IEEEeqnarray}
\end{subequations}
{\color{black}The proposed framework addresses mobility and time-varying channel conditions through uncertainty-aware training and sensing-guided adaptation rather than explicit temporal modeling. Specifically, mobility effects are captured during training by exposing the network to perturbed radar echoes and imperfect CSI realizations that reflect variations in angles of arrival, channel coefficients, and sensing noise. During operation, the friendly jamming beam is adapted based on instantaneous radar echo statistics and the associated Fisher information estimates, which are recomputed from current sensing observations. As a result, although the neural network parameters are trained offline, the jamming strategy adapts implicitly to environmental changes without assuming a specific temporal correlation model or relying on online learning or recurrent neural network architectures.}

\subsection{Quantized Tensor-Train Based Neural Compression}
In high-dimensional learning-based ISAC systems, neural networks used to learn mappings from CSI or radar features to beamforming weights often contain FC layers with millions of parameters. While effective, such models face scalability and overfitting challenges, especially when trained with imperfect or partial channel state information. To address this, we leverage tensor decomposition to compress large weight matrices in the encoder, reducing both memory footprint and training complexity without sacrificing learning capability \cite{sidiropoulos2017tensor}.

Let the input to the encoder be a vectorized tensor \( \mathbf{x} \in \mathbb{R}^{N} \), from reshaped CSI of size \( N_t \times N_s \times T \), and let the output be a vector \( \mathbf{z} \in \mathbb{R}^{M} \) used to generate precoder and FJ weights. The conventional FC-layer mapping is defined as follows:
\begin{equation}
\mathbf{z} = \mathbf{W} \mathbf{x} + \mathbf{b}, \quad \text{where } \mathbf{W} \in \mathbb{R}^{M \times N}.
\end{equation}
In the TT-based approach, the weight matrix 
$\mathbf{W} \in \mathbb{R}^{M \times N}$ is reparameterized 
as a sequence of low-rank tensor cores is given as follows:
\begin{equation}
\mathcal{W}(m_1,\ldots,m_q,n_1,\ldots,n_q) 
= \prod_{k=1}^{q} G_k[m_k,n_k]
\end{equation}
where the input and output dimensions are factorized as 
$N = \prod_{i=1}^{q} n_i$ and $M = \prod_{i=1}^{q} m_i$, 
$G_k[m_k,n_k] \in \mathbb{R}^{r_{k-1} \times r_k}$ denotes the $k$-th TT core, 
and $\{r_k\}_{k=0}^{q}$ are the TT ranks with $r_0 = r_q = 1$. 
The total number of trainable parameters is presented as follows:
\begin{equation}
\sum_{k=1}^{q} r_{k-1} m_k n_k r_k \;\;\ll\;\; MN
\end{equation}
which demonstrates the significant reduction in model complexity compared to the original dense parameterization.

In our multicarrier TT-based framework, the encoder operates on a per-subcarrier feature vector 
\begin{align}
\mathbf{z}^{(n)} = \Big[
& \mathrm{Re}\!\big(\mathrm{vec}(\hat{\mathbf{H}}_B^{(n)})\big), \;
  \mathrm{Im}\!\big(\mathrm{vec}(\hat{\mathbf{H}}_B^{(n)})\big), \notag \\[2pt]
& \gamma^{(n)}, \; \zeta^{(n)}, \; \theta, \; \phi, \; n/N
\Big]^T 
\end{align}
where $\hat{\mathbf{H}}_B^{(n)}$ is the estimated downlink CSI on subcarrier $n$, 
$\gamma^{(n)}$ and $\zeta^{(n)}$ are the communication and jamming power coefficients, 
and $(\theta,\phi)$ denote angular parameters. 
From this input, the encoder produces the user beamformers $\{\mathbf{f}_k^{(n)}\}$ and a candidate FJ beam $\mathbf{v}^{(n)}$, which are subsequently screened via Fisher-information and CRLB constraints. 
The encoder variant employed in this work is the Deep TT Encoder (DTTE), a hybrid design consisting of a compressed TT core followed by an FC, making a balance between the scalability of tensor decomposition and the expressive capability of FC networks.

{\color{black}

Quantization is applied after tensor train factorization by representing each tensor train core using reduced numerical precision. Specifically, each element $g$ in a tensor train core $\mathbf{G}_k$ is mapped to a quantized value $\tilde{g}$ according to a uniform quantization rule
\[
\tilde{g} = Q(g) = \Delta \cdot \mathrm{round}\!\left(\frac{g}{\Delta}\right),
\]
where $\Delta$ denotes the quantization step size. The resulting set of quantized tensor train cores constitutes the quantized tensor train encoder referred to in the Introduction. This quantization step is applied after offline training and does not alter the network architecture or learning objective, while significantly reducing memory footprint and inference complexity.
}

{\color{black}
In addition to model size reduction, we evaluate the computational complexity of the proposed framework in terms of floating-point operations (FLOPs) per inference. Owing to the quantized tensor-train representation, the encoder exhibits significantly reduced inference complexity compared to its uncompressed counterpart, with a fixed and predictable FLOPs count that does not depend on online interaction or iterative optimization. This lightweight inference characteristic makes the proposed approach suitable for real-time ISAC deployment under practical latency and computational constraints.
}

{\color{black}
\subsection{Hardware Impairments}

In this section, we explicitly model two primary sources of Radio Frequency (RF) chain impairments: Phase Noise (PN) and I/Q Imbalance. We model the phase noise $\phi(n)$ as a discrete Wiener process, where the received signal becomes $y(n) = e^{j\phi(n)} \mathbf{h}^H \mathbf{x}(n) + w(n)$, with $\phi(n) \sim \mathcal{N}(0, \sigma_{PN}^2)$. Similarly, we model I/Q imbalance by introducing amplitude mismatch $\epsilon$ and phase mismatch $\Delta\theta$, distorting the transmitted symbol $x$ into $\tilde{x} = \mu x + \nu x^*$, where $\mu = \cos(\Delta\theta/2) + j\epsilon \sin(\Delta\theta/2)$ and $\nu = \epsilon \cos(\Delta\theta/2) - j \sin(\Delta\theta/2)$.

We evaluate the impact of RF front-end non-idealities on the proposed secure ISAC framework. Following standard complex-baseband models, phase noise is modeled as a multiplicative distortion
\[
\mathbf{y}[n] = e^{j\phi[n]}\mathbf{H}\mathbf{x}[n] + \mathbf{w}[n],
\]
where $\phi[n]$ follows a discrete-time Wiener phase process. I/Q imbalance is modeled via a widely-linear transmit distortion
\[
\tilde{\mathbf{x}}[n] = \boldsymbol{\mu}\,\mathbf{x}[n] + \boldsymbol{\nu}\,\mathbf{x}^*[n],
\qquad
\mathbf{y}[n] = \mathbf{H}\tilde{\mathbf{x}}[n] + \mathbf{w}[n],
\]
where $\boldsymbol{\mu},\boldsymbol{\nu}$ capture amplitude and phase mismatch between the I and Q branches.
}

\subsection{Neural ISAC-Based Multicarrier FJ Training Workflow}
To enable end-to-end optimization across subcarriers, we adopt a neural autoencoder framework. The encoder maps subcarrier-wise estimated CSI \( \hat{\mathbf{H}}_B^{(n)} \) to beamforming vectors \( \{\mathbf{f}_k^{(n)}\} \), while candidate FJ beams  \( \mathbf{v}_i^{(n)} \) are evaluated based on FIM sensitivity.

For multicarrier beamforming and FJ learning, each subcarrier 
$n$ is represented by a real-valued feature vector:
\begin{align}
\mathbf{z}^{(n)} = \text{concat}\Big(
&\mathrm{Re}(\mathrm{vec}(\hat{\mathbf{H}}_B^{(n)})), 
\mathrm{Im}(\mathrm{vec}(\hat{\mathbf{H}}_B^{(n)})), \notag\\
&\gamma^{(n)}, \zeta^{(n)}, \theta, \phi, n/N
\Big)
\end{align}
capturing both communication and sensing context. The neural encoder processes the full feature set \( \mathbf{Z} \in \mathbb{R}^{N \times D} \) to jointly generate user beamformers \( \{\mathbf{f}_k^{(n)}\} \) and FJ directions \( \{\mathbf{v}^{(n)}\} \). Radar echoes are then used to estimate the Fisher Information and CRLBs, guiding loss minimization under power and angular accuracy constraints.

The training loss aggregates communication throughput and FIM estimation quality over all subcarriers:
\begin{equation}
\mathcal{L}_{\text{total}} = -\sum_{n=1}^N \sum_{k=1}^K R_k^{(n)} + \lambda \sum_{n=1}^N \left( \text{CRLB}_\theta^{(n)} + \text{CRLB}_\phi^{(n)} \right).
\end{equation}

In the post-training phase, the BS selects \( \mathbf{v}^{(n)} \) per subcarrier that:
$i)$ lies in \( \mathcal{N}(\hat{\mathbf{H}}_B^{(n)}) \),
$ii)$ satisfies angular estimation constraints, and
$iii)$ maximizes subcarrier FIM trace \( \mathrm{Tr}(\mathbf{J}^{(n)}) \).
This multicarrier-aware jamming strategy leverages spatial diversity and frequency selectivity to jointly enhance estimation accuracy and communication secrecy.

\section{Friendly Jamming for Non-Overlapping Carrier ISAC}
{\color{black}
Section \ref{sec:multicarrier} considers an overlapping multicarrier ISAC framework in which all subcarriers are jointly used for communication, sensing, and friendly jamming, enabling full spectral reuse and flexible jamming design under CRLB-based sensing constraints. In contrast, Section IV investigates a non-overlapping multicarrier ISAC architecture as a structured extension motivated by practical interference management requirements.

In the non-overlapping design, subcarriers are partitioned into disjoint sets dedicated to communication and sensing and jamming, respectively. While this separation reduces mutual interference, it significantly limits the spectral degrees of freedom available for friendly jamming and alters the trade-off between secrecy enhancement and sensing accuracy. 
}
\label{sec:nonoverlap}
\subsection{FJ for Non-Overlap Multicarrier ISAC System Model}
In this design, the available subcarriers are partitioned into two disjoint sets: communication-dedicated subcarriers $\mathcal{N}_c$ and sensing-and-jamming subcarriers $\mathcal{N}_s$, such that $\mathcal{N}_c \cap \mathcal{N}_s = \emptyset$ and $\mathcal{N}_c \cup \mathcal{N}_s = \{1, 2, \ldots, N\}$. This separation ensures that mutual interference between information-bearing and jamming signals is inherently avoided across subcarriers.
For each communication subcarrier $n \in \mathcal{N}_c$, the achievable rate of user $k$ treating multi-user interference and jamming as noise and the information leakage rate to Eve are given in \eqref{Rates_FJ_Non}.
\begin{figure*}[!t]
\begin{equation}
\label{Rates_FJ_Non}
\begin{split}
R_k^{(n)} &= \bar{\tau} \log_2 \Bigg( 
1 + \frac{
\gamma_k^{(n)} \left| (\mathbf{h}_k^{(n)})^H \mathbf{f}_k^{(n)} \right|^2
}{
\sum\limits_{j \ne k} \gamma_j^{(n)} \left| (\mathbf{h}_k^{(n)})^H \mathbf{f}_j^{(n)} \right|^2 
+ \zeta^{(n)} \left| (\mathbf{h}_k^{(n)})^H \mathbf{v}^{(n)} \right|^2 
+ \sigma_c^2
}\Bigg), \\
R_{E,k}^{(n)} &=\bar{\tau} \log_2 \det \Bigg( 
\mathbf{I} + 
\Big( \zeta^{(n)} \mathbf{H}_E^{(n)H} \mathbf{v}^{(n)} \mathbf{v}^{(n)H} \mathbf{H}_E^{(n)} 
+ \sigma_E^2 \mathbf{I} \Big)^{-1}
\mathbf{H}_E^{(n)H} \mathbf{f}_k^{(n)} \mathbf{f}_k^{(n)H} \mathbf{H}_E^{(n)}
\Bigg).
\end{split}
\end{equation}
\hrulefill
\vspace*{-3mm}
\end{figure*}
The denominator captures the cumulative impact of inter-user interference and jamming leakage on subcarrier $n$.
The per-subcarrier secrecy rate for user $k$ on subcarrier $n \in \mathcal{N}_c$ is given by the positive difference:
\begin{align}
R_{\mathrm{sec},k}^{(n)} = \left[ R_k^{(n)} - R_{E,k}^{(n)} \right]^+.
\end{align}

\subsection{Problem Formulation of FJ for Non-overlap Multi-carrier ISAC}
The goal is to jointly optimize the power allocation across communication and sensing subcarriers to maximize the secure communication throughput while maintaining angular estimation performance. The optimization problem is defined as follows:
\begin{subequations} \label{eq:problem43}
	\begin{IEEEeqnarray}{cl}
		\ &\underset{\{\gamma_k^{(n)}\}, \{\zeta^{(n)}\}}{\mathrm{max}} \quad   \sum_{n=1}^{N} \sum_{k=1}^{K}
\left[ R_k^{(n)} - R_{E,k}^{(n)} \right]^+
		\label{eq:objective_39} \\
		& \st\  \sum_{n \in \mathcal{N}_c} \sum_{k=1}^{K} \gamma_k^{(n)} \xi_k^{(n)}
  + \sum_{n \in \mathcal{N}_s} \zeta^{(n)} \leq P_{\max}
\label{eq:constraint_39a} \\
&\quad\,\ \sum_{n \in \mathcal{N}_s} \mathrm{CRLB}_\theta^{(n)}
  \leq \mathrm{CRLB}_\theta^{\max}
\label{eq:constraint_39b} \\
&\quad\,\ \sum_{n \in \mathcal{N}_s} \mathrm{CRLB}_\phi^{(n)}
  \leq \mathrm{CRLB}_\phi^{\max}.
\label{eq:constraint_39c} 
	\end{IEEEeqnarray}
\end{subequations}
For each $n \in \mathcal{N}_s$, the Fisher information associated with angular parameters $\omega \in \{\theta, \phi\}$ is given by
\begin{align}
J_\omega^{(n)} = \frac{2 \zeta^{(n)}}{\sigma_s^2} 
\left\| \frac{\partial \mathbf{G}(\theta, \phi) \mathbf{v}^{(n)}}{\partial \omega} \right\|^2
\end{align}
which determines the per-subcarrier CRLB via $\mathrm{CRLB}_\omega^{(n)} = 1 / J_\omega^{(n)}$. These angular CRLBs serve as performance constraints for sensing fidelity in (C2') and (C3').

To address the above problem, the training and estimation procedures, along with the beamforming architecture for multicarrier ISAC, follow the framework outlined in Section \ref{sec:multicarrier}. The use of non-overlapping subcarrier assignment simplifies interference management by decoupling communication and jamming functions. Although this partitioning reduces the number of subcarriers available for communication, it enhances robustness under imperfect CSI by isolating jamming and radar transmissions from user data. This separation enables adaptive beam design that improves both security and sensing accuracy.


\begin{table}[ht]
\centering
\caption{Simulation Parameters}
\label{tab:Sim_Setup}
\renewcommand{\arraystretch}{1.3}
\begin{tabular}{l|l}
\hline
\textbf{Parameter} & \textbf{Value} \\ \hline\hline
$N_t$, $N_r$, $N_e$ & 16, 4, 2 \\ \hline
Number of users, $K$ & 2 \\ \hline
Number of subcarriers, $N$ & 64 \\ \hline
Set of comm subcarriers, $\mathcal{N}_c$ & First 32 subcarriers (1--32) \\ \hline
Set of sensing/jamming, $\mathcal{N}_s$ & Remaining 32 subcarriers (33--64) \\ \hline
Modulation scheme & QPSK ($M = 4$) \\ \hline
Channel model & i.i.d. Rayleigh fading \\ \hline
CSI error variance & $\rho \in \{0, 0.05, 0.1, 0.2\}$ \\ \hline
Noise power & $\sigma_n^2 = 1$ (normalized) \\ \hline
Transmit power budget & $P_{\max} = 30$ dB \\ \hline
CRLB constraint & $\text{CRLB}_\theta, \text{CRLB}_\phi \leq -30$ dB \\ \hline
Sensing target angles & $(\theta, \phi) = (10^\circ, 15^\circ)$ \\ \hline
Optimizer & Adam \\ \hline
Learning rate and batch size & $10^{-3}$ and $128$ \\ \hline
Autoencoder Encoder & \makecell{FC ($N_t \rightarrow 128$) $\rightarrow$ ReLU \\ $\rightarrow$ FC ($128 \rightarrow N_t$) $\rightarrow$ BatchNorm} \\ \hline
Autoencoder Decoder & \makecell{FC ($1 \rightarrow 128$) $\rightarrow$ ReLU \\ $\rightarrow$ FC ($128 \rightarrow M$) $\rightarrow$ Softmax} \\ \hline
FIM estimator & \makecell{FC (echo dim $\rightarrow 128$) $\rightarrow$ ReLU \\ $\rightarrow$ FC ($128 \rightarrow 1$)} \\ \hline
Loss function (comm) & Cross-Entropy (symbol classification) \\ \hline
Loss function (sensing) & $f$-divergence via Donsker--Varadhan bound \\ \hline
\end{tabular}
\end{table}

\section{Numerical Results}
\label{sec:Numerical_Results}
{\color{black}
The dataset used in this work is synthetically generated based on the physical models described in Sections~II and III. The communication channels follow a flat Rayleigh fading model with additive white Gaussian noise, where the entries of the legitimate and eavesdropper channel matrices satisfy $[\mathbf{H}_B]_{i,j}, [\mathbf{H}_E]_{i,j} \sim \mathcal{CN}(0,1)$. The imperfect CSI is modeled by adding Gaussian estimation noise to the legitimate channels, with variances $\rho \in \{0, 0.05, 0.1, 0.2\}$. Radar echo data are generated by perturbing the azimuth and elevation angles. Specifically, for each probing beam, the perturbations $\boldsymbol{\delta}_i \sim \mathcal{N}(0, 0.05\,\mathbf{I}_3)$ are identically applied to both angular parameters, yielding perturbed estimates $(\theta + \delta_{i1}, \phi + \delta_{i2})$. Since the dimension of the estimated parameter vector is two, the perturbation dimension is set to $D=4$. For each perturbation pair, $5000$ samples are generated to train the nonparametric Fisher information matrix estimator.

The training process follows a two-stage neural autoencoder framework implemented using the Adam optimizer with a learning rate of $10^{-3}$ and a batch size of $128$. In Stage~1, the communication encoder is optimized by minimizing a combination of categorical cross-entropy and rate loss to learn the communication beamformers. In Stage~2, the friendly jamming beams are optimized by maximizing angular sensitivity through the estimated Fisher information matrix while strictly enforcing Cramér--Rao lower bound constraints.
The proposed ISAC-friendly jamming framework is evaluated against the baseline methods in \cite{liu2024artificial} and \cite{tuan2025securing} in terms of the sum secrecy rate. The main simulation parameters are summarized in Table~\ref{tab:Sim_Setup}.
}

\begin{figure}[!htb]
\centering
\includegraphics[width=1\linewidth]{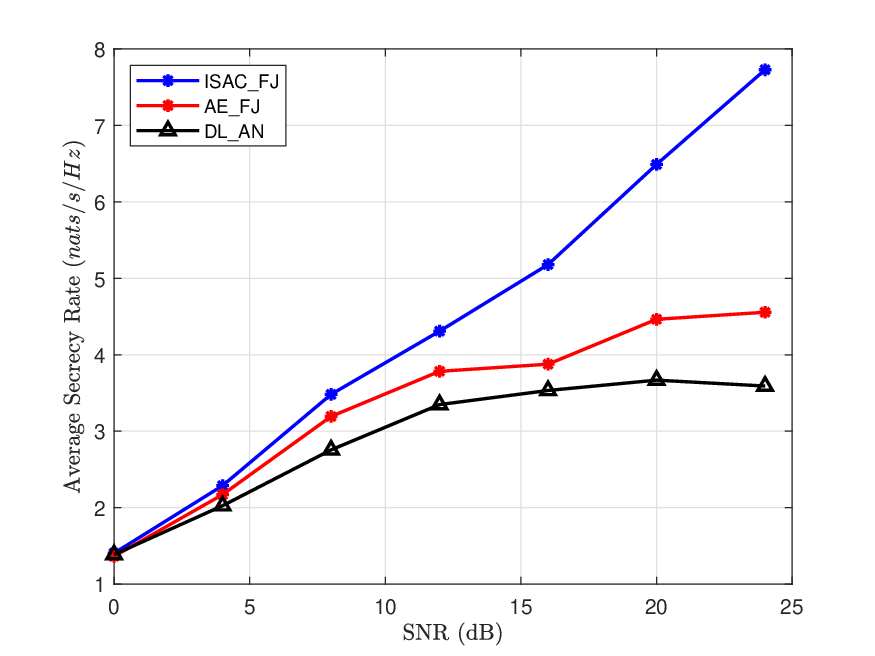}
\caption{Sum secrecy rate comparison with perfect AoA estimation.}
\label{fig:Im_CSI_COMPARE_ISAC}
\end{figure}

Fig.~\ref{fig:Im_CSI_COMPARE_ISAC} shows that the proposed ISAC\_FJ scheme significantly outperforms AE\_FJ and DL\_AN in terms of average secrecy rate across all power levels with the assumption of perfect AoA estimation. For instance, at $\text{SNR} = 20$ dB, ISAC\_FJ achieves approximately 6.9 nats/s/Hz, while AE\_FJ and DL\_AN attain only 4.5 and 3.6 nats/s/Hz, respectively, reflecting relative gains of over $50\%$ and $90\%$.
This improvement is due to ISAC\_FJ's sensing-guided beam design, which focuses jamming power in directions most likely aligned with Eve, rather than spreading it over the entire null space. AE\_FJ and DL\_AN distribute jamming more broadly, reducing power efficiency and limiting the secrecy rate.
This demonstrates a clear trade-off: wide-direction jamming improves coverage but reduces the power available for secure communication. In contrast, ISAC\_FJ reduces jamming dimensionality and improves power utilization, yielding higher secrecy performance, particularly in high SNR regimes.

\begin{figure}[!htb]
\centering
\includegraphics[width=1\linewidth]{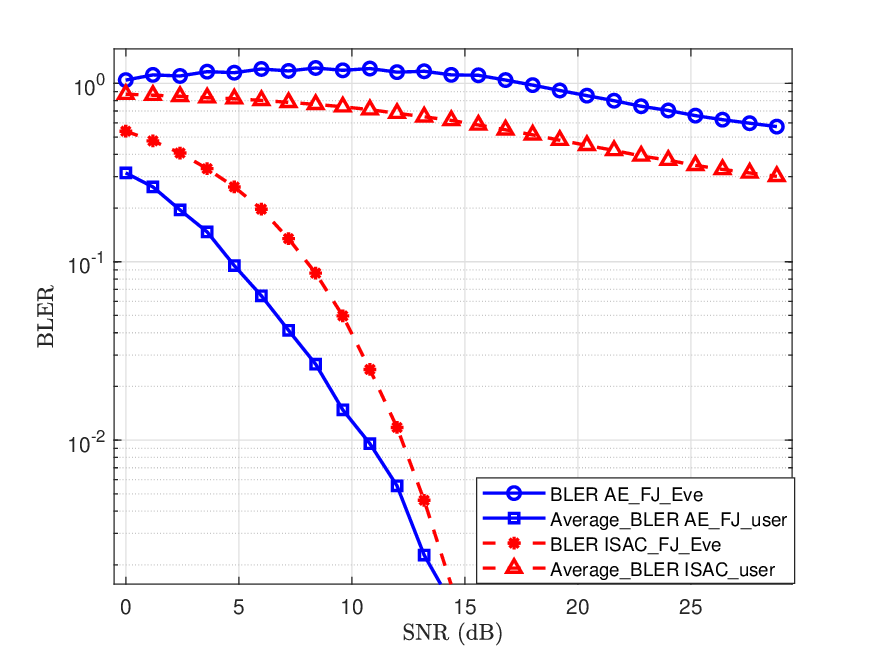}
\caption{BLER on the legitimate channel.}
\label{fig:BLER_Compare}
\end{figure}
The transmission reliability for legitimate users and the decoding performance of Eve, assuming Eve employs an MMSE decoder, are shown in Fig.~\ref{fig:BLER_Compare}. The figure presents the average BLER at the three users and at Eve for both ISAC\_FJ and the AE\_FJ baseline. While AE\_FJ achieves slightly lower BLER for legitimate users, ISAC\_FJ still ensures reliable communication, maintaining BLER below $10^{-2}$ at $\text{SNR} = 15$ dB. More importantly, ISAC\_FJ significantly degrades Eve’s decoding ability, keeping her BLER near 1 across all power levels. In contrast, AE\_FJ allows Eve's BLER to decline as transmit power increases. These results demonstrate that ISAC\_FJ effectively balances reliability and security by focusing jamming power in threat-relevant directions.

\begin{figure}[!htb]
\centering
\includegraphics[width=1\linewidth]{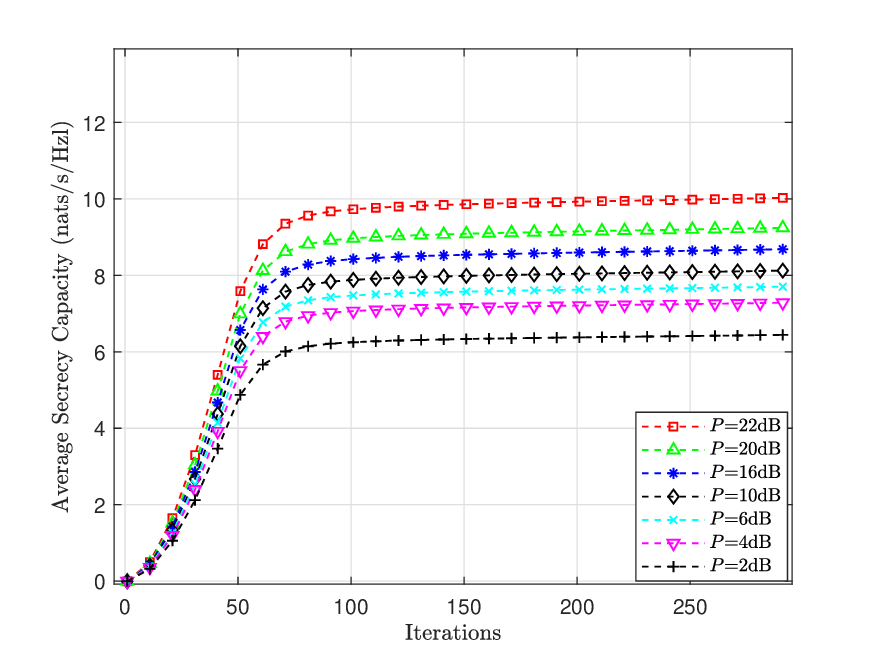}
\caption{Convergence of training process.}
\label{fig:Converge_Av_sec}
\end{figure}
Fig.~\ref{fig:Converge_Av_sec} illustrates the convergence behavior of the proposed secrecy optimization. The method shows rapid convergence within the first $50$ iterations, followed by stable performance. As expected, higher transmit power budgets yield higher sum secrecy rates, with the sum secrecy rate saturating around $10$ Nats/symbol at $P = 22$ dB. The flat curves beyond iteration $100$ confirm the stability and convergence of the end-to-end learning framework across different power levels.

\begin{figure}[!htb]
\centering
\includegraphics[width=1\linewidth]{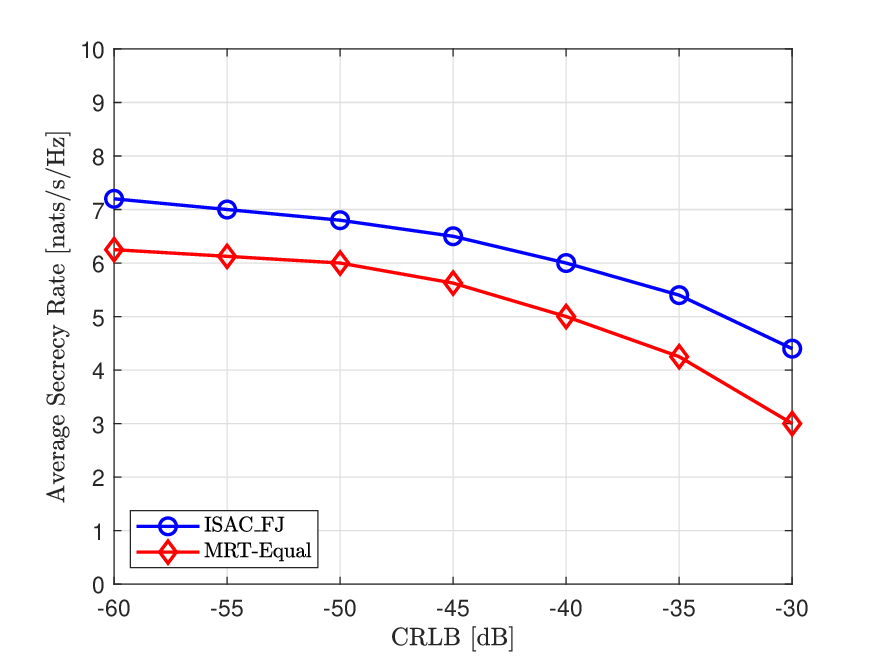}
\caption{Average secrecy rate versus CRLB.}
\label{fig:Sec_CRLB}
\end{figure}

Fig.~\ref{fig:Sec_CRLB} illustrates the relationship between the sum secrecy rate and the CRLB, defined as the mean of $\text{CRLB}_{\theta}$ and $\text{CRLB}_{\phi}$.  The proposed ISAC\_FJ method consistently outperforms the MRT-Equal baseline from \cite{nguyen2025performance}, which uses a closed-form steering vector $\textbf{G}(\theta, \phi)$ across all CRLB levels. As CRLB values decrease, both methods achieve higher secrecy rates. However, as CRLB increases toward $-30$ dB, indicating degraded sensing accuracy, MRT-Equal’s performance drops more sharply. ISAC\_FJ remains robust, maintaining a secrecy rate above $4$ nats/s/Hz, while MRT-Equal falls to around $3$ nats/s/Hz. This demonstrates ISAC\_FJ’s resilience to sensing errors and its advantage in imperfect conditions. Importantly, ISAC\_FJ achieves this without relying on a closed-form estimation function, leveraging data-driven end-to-end learning guided by CRLB constraints instead.


To evaluate the advantage of multicarrier design, we compare the proposed ISAC-FJ scheme in a single-carrier and a multicarrier setting under the same total transmit power constraint \( P_{\max} = 30 \). In the single-carrier scheme, all power is allocated to a single frequency slot, while the power allocation is distributed across a certain subcarriers in the multicarrier scheme.  
Fig. \ref{fig:Worstcase_5curves} evaluates the average worst-user secrecy rate as a function of SNR and the number of subcarriers $N$. The results demonstrate the superiority of the multicarrier approach. The single-carrier system ($N=1$) achieves a secrecy rate of nearly zero across all SNRs, indicating its inability to secure the channel.
Performance scales directly with the number of subcarriers. At an SNR of 25 dB, the secrecy rate increases from approximately 1.6 nats/s/Hz for $N=16$ to 4.3 nats/s/Hz for $N=32$. The advantage is even more pronounced for $N=64$, which achieves approximately 6.2 nats/s/Hz at the same SNR. This substantial improvement is attributed to the model's ability to leverage frequency diversity, flexibly allocating power across subcarriers to exploit channel advantages. The strong, unsaturated growth of the $N=64$ curve, reaching over 7 nats/s/Hz at 30 dB, confirms the scalability and effectiveness of the proposed framework for wideband secure ISAC.

\begin{figure}[!htb]
\centering
\includegraphics[width=1\linewidth]{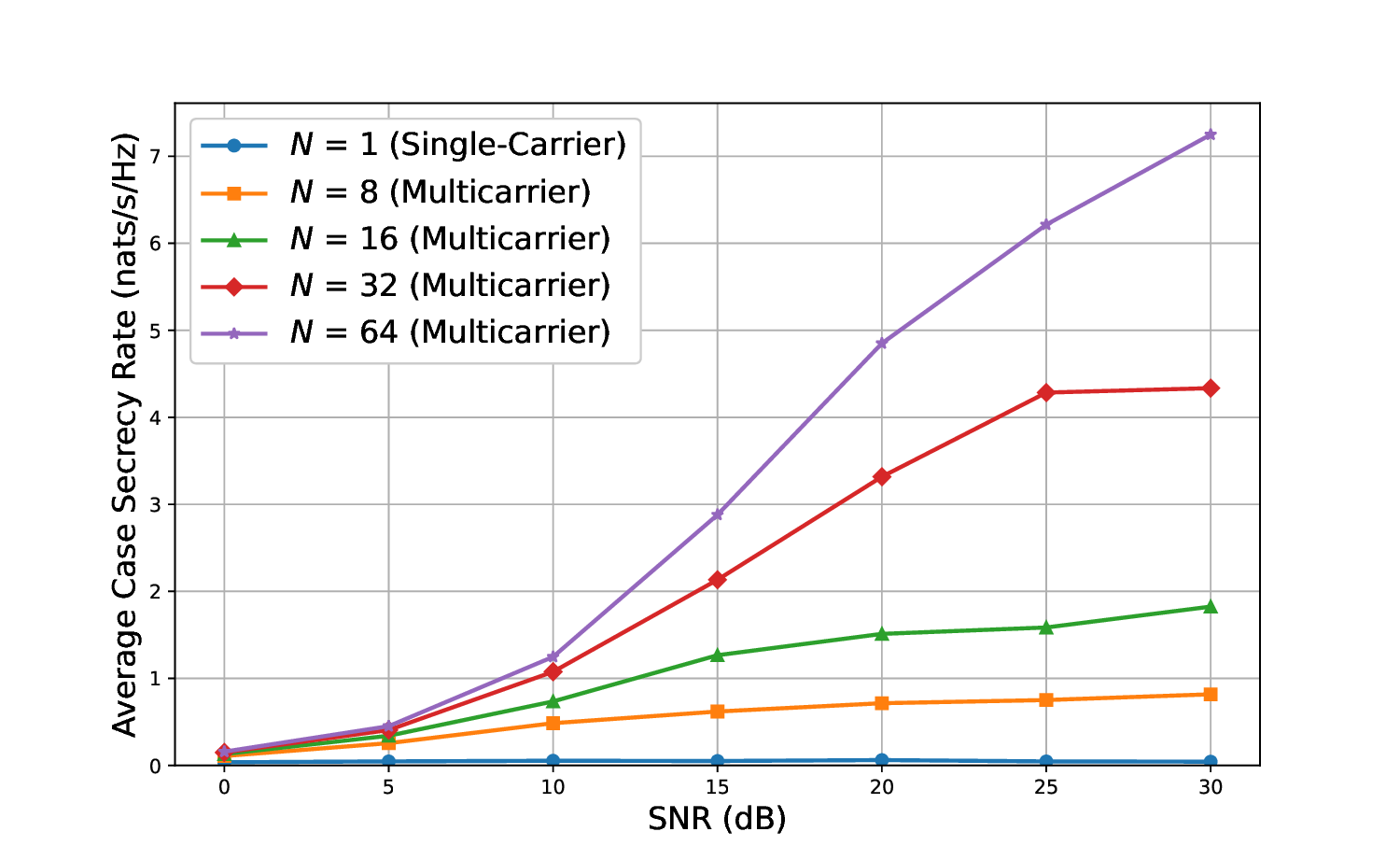}
\caption{Average secrecy rate in worst case per user's channel.}
\label{fig:Worstcase_5curves}
\end{figure}

We investigate the impact of channel estimation errors on the secrecy rate. Fig.~\ref{fig:Sec__SNR_pertub} shows the sum secrecy rate versus different levels of CSI error variance $\rho \in \{0, 0.05, 0.1, 0.2\}$. As $\rho$ increases, the secrecy rate degrades due to uncertainty in the legitimate channel. However, the degradation is gradual, indicating that the proposed model is resilient to moderate CSI errors, unlike traditional AN-based methods, which fail under imperfect CSI.

\begin{figure}[!htb]
\centering
\includegraphics[width=1\linewidth]{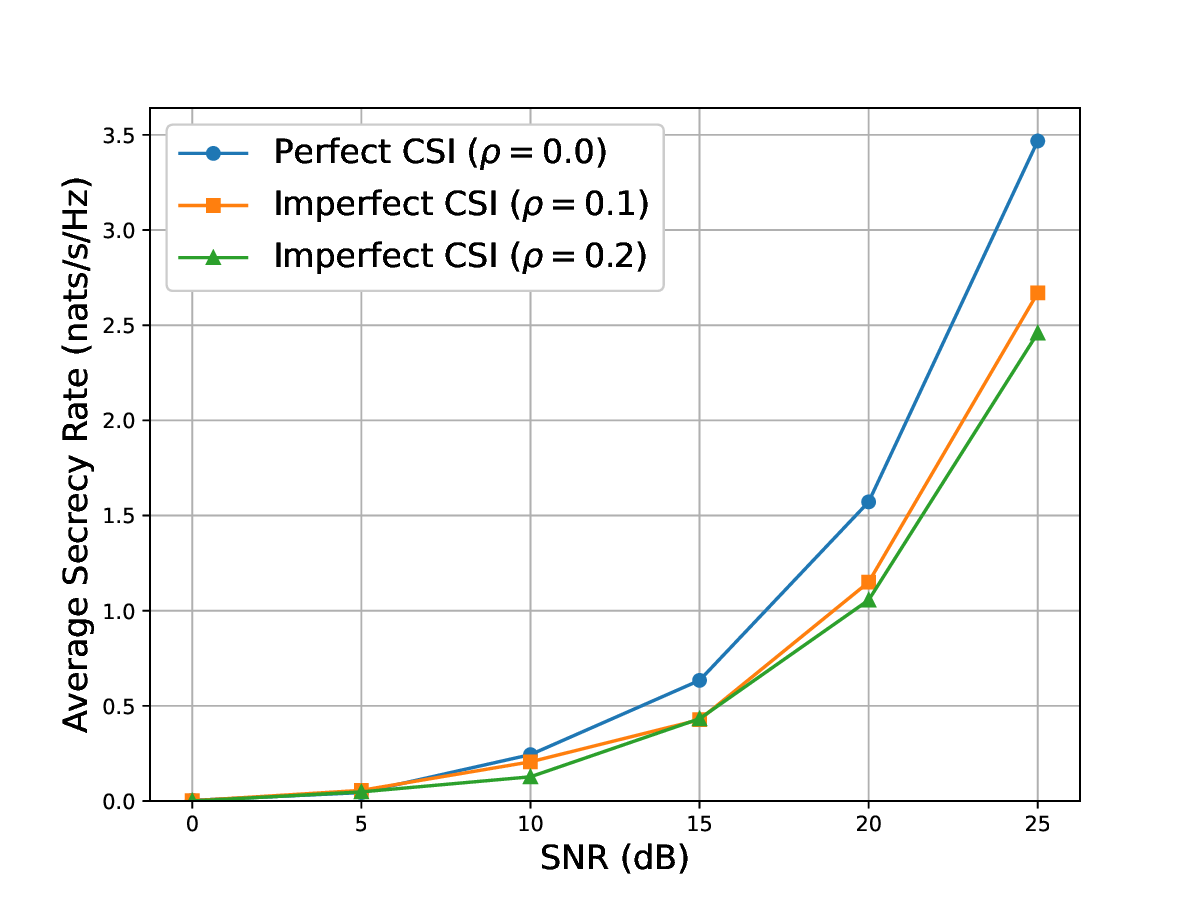}
\caption{Average secrecy rate versus SNR with different perturbation levels of estimation.}
\label{fig:Sec__SNR_pertub}
\end{figure}

Fig.~\ref{fig: Performance_Comparison} illustrates the performance–efficiency trade-off among the FC encoder, the pure TT encoder, and the proposed Hybrid TT-FC encoder. The pure TT model, implemented with a single layer under aggressive compression (low TT-rank), achieves a parameter reduction of over two orders of magnitude; however, this comes at the cost of a secrecy rate degradation exceeding $10\%$, attributable to its limited learning capacity.
In contrast, the Hybrid TT-FC encoder addresses this limitation by combining a compressed TT core with two standard FC layers as the output head. The TT core efficiently processes the high-dimensional channel input using low TT rank tensors, maintaining parameter efficiency. The FC head compensates for the lost capacity by enhancing nonlinearity and representational power. This hybrid architecture restores performance: at 30 dB SNR, it achieves over $99\%$ of the secrecy rate of the FC baseline with less than $1\%$ loss, while still being significantly smaller in parameter count.
Training convergence is shown in Fig.~\ref{fig:Loss_Convergence}. The pure TT model converges slowly due to under-parameterization. In contrast, the Hybrid TT-FC model, trained with 80 epochs, demonstrates stable and faster convergence, benefiting from the added flexibility of the FC head.

\begin{figure}[!htb]
\centering
\includegraphics[width=1\linewidth]{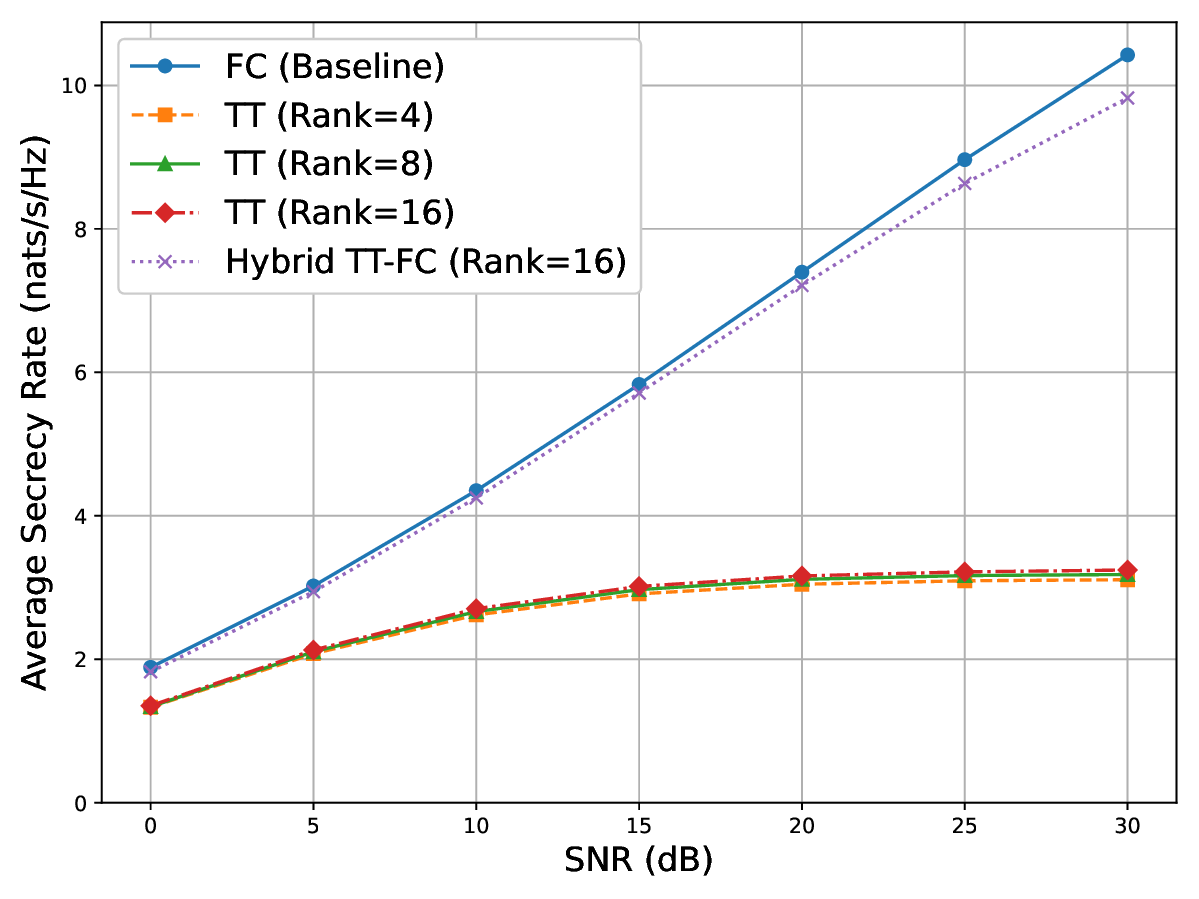}
\caption{Comprehensive comparison of the average secrecy rates between FC encoder and quantised TT.}
\label{fig: Performance_Comparison}
\end{figure}

\begin{figure}[!htb]
\centering
\includegraphics[width=1\linewidth]{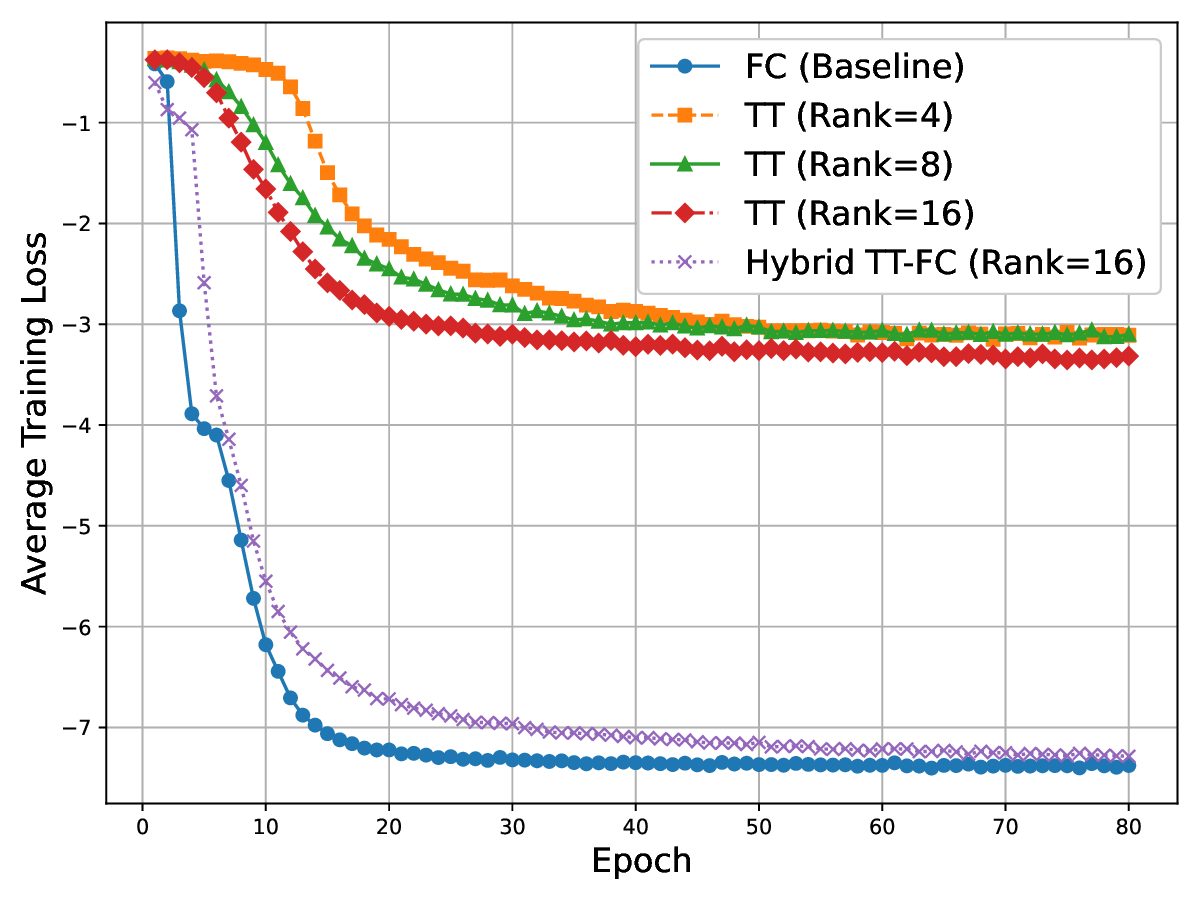}
\caption{Comparison of average secrecy rates and training loss between FC and TT method.}
\label{fig:Loss_Convergence}
\end{figure}

{\color{black}
\begin{table*}[!ht]
\label{tab:architecture_comparison}
\centering
\caption*{TABLE III: Comparison of Encoder Architectures: FC, DTTE, and Pure TT}
\small
\resizebox{\textwidth}{!}{%
\begin{tabular}{|l|c|c|c|}
\hline
\textbf{Metric} & \textbf{FC Encoder} & \textbf{DTTE} & \textbf{Pure TT Encoder} \\
\hline
\textbf{Architecture} & Multi-layer FC network & TT core + FC head & Single TT-matrix layer \\
\hline
\textbf{Compression} & None & Hybrid TT + FC & Full TT decomposition \\
\hline
\textbf{Parameters} ($N_t{=}16$, $K{=}2$) & 62,563 & \textbf{28,371} & 2,800 \\
\hline
\textcolor{black}{\textbf{FLOPs (Inference)}} & \textcolor{black}{$\approx$ \textbf{125,000}} & \textcolor{black}{$\approx$ \textbf{57,000}} & \textcolor{black}{$\approx$ \textbf{5,600}} \\
\hline
\textbf{Secrecy Rate} & Highest (baseline) & Near-baseline & Lower due to overcompression \\
\hline
\textbf{Inference Time} & Slowest & Moderate & Fastest \\
\hline
\textbf{Trade-off} & Best performance; high cost & Balanced performance and size & High compression; degraded accuracy \\
\hline
\end{tabular}%
}
\end{table*}
}

Fig.~\ref{fig:Over_lap_channel_sec_rate} compares the worst-user channel and secrecy rates for different hybrid subcarrier allocations, revealing a significant trade-off between throughput and security. While dedicating more subcarriers to communication boosts the channel rate, it can degrade the secrecy rate.
At an SNR of $30$ dB, the $50\%$ Comm-Only scheme attains the highest channel rate 
($\approx 33$ nats/s/Hz) but suffers from the lowest secrecy rate 
($\approx 4.8$ nats/s/Hz). The optimal configuration is achieved with the 
$25\%$ Comm-Only scheme, which yields the highest secrecy rate. This outperforms both the fully overlapping scheme ($0\%$), which reaches only $7.2$ nats/s/Hz due to the presence of friendly jamming on all $64$ subcarriers that imposes a ceiling on secrecy performance, and the $50\%$ scheme. The degradation in the $50\%$ allocation arises because it must allocate stronger jamming power to fewer joint carriers to satisfy the sensing CRLB constraints, thereby exacerbating information leakage. By contrast, the $25\%$ scheme strikes the most effective balance between preserving jam-free communication carriers.

\begin{figure}[!htb]
\centering
\includegraphics[width=1.1\linewidth]{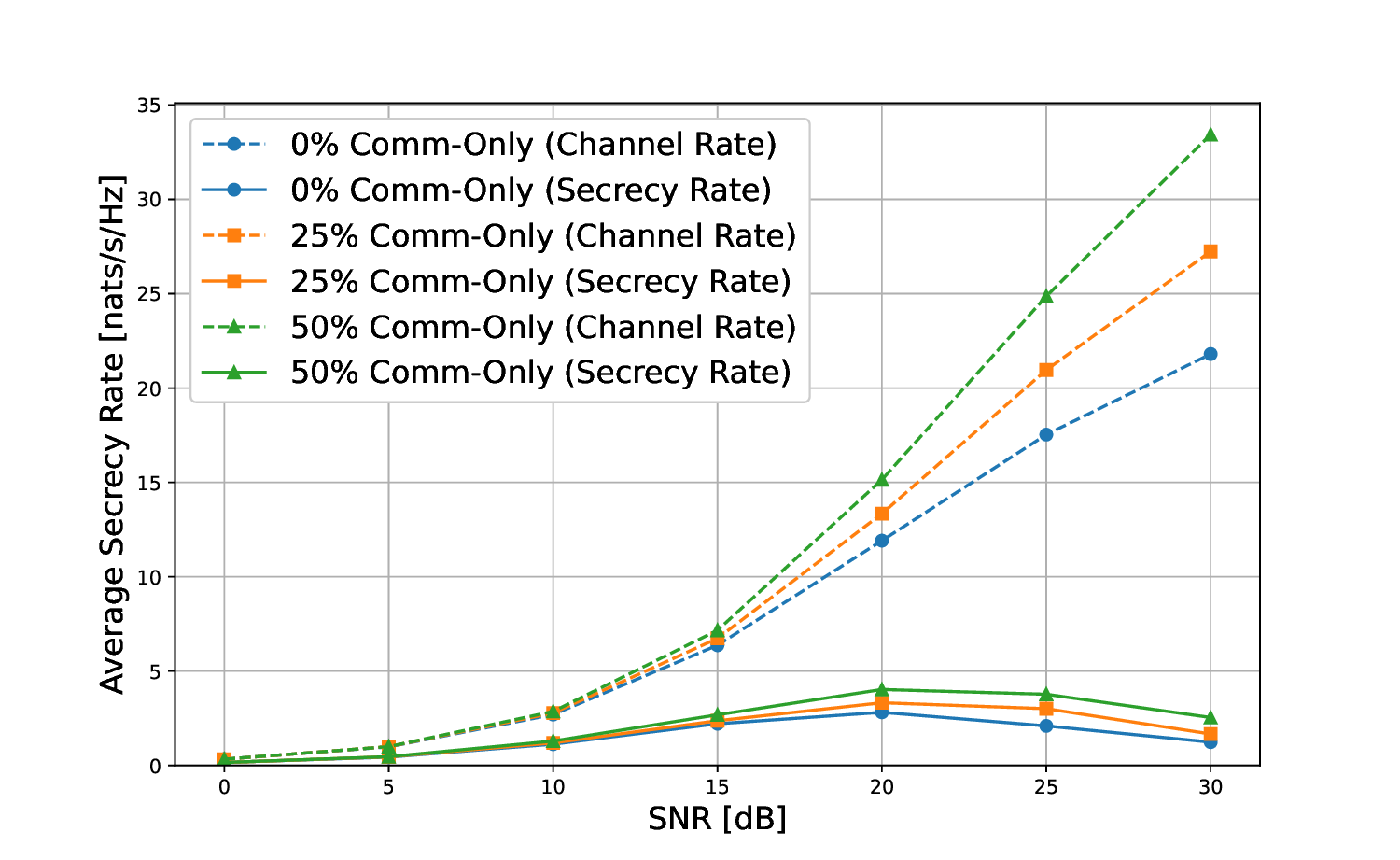}
\caption{Average secrecy rate versus SNR in different levels of non-overlapping.}
\label{fig:Over_lap_channel_sec_rate}
\end{figure}

{\color{black}
\begin{figure}[!htb]
\centering
\includegraphics[width=.8\linewidth]{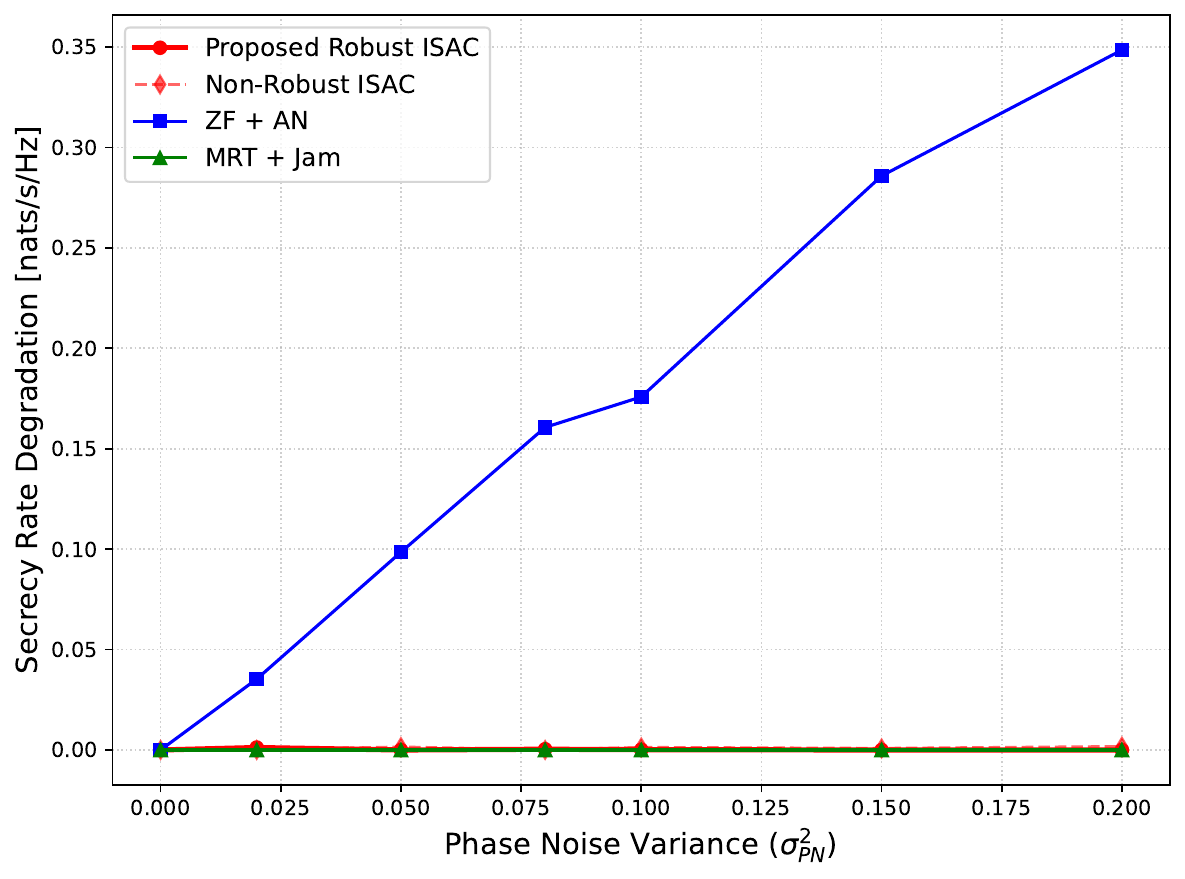}
\caption*{{\color{black}
Fig. 12: Secrecy rate degradation, defined as the loss relative to the ideal hardware baseline.}}
\label{fig:Figure_Secrecy_Rate_Degradation2}
\end{figure}

Fig. 12 illustrates the secrecy-rate degradation, defined as the performance loss relative to the ideal hardware baseline, under amplitude mismatch $\epsilon=0.02$, phase error $\Delta\theta=1^{\circ}$, and increasing phase noise variance $\sigma_{\text{PN}}^2 \in [0, 0.2]$. The proposed Robust ISAC scheme (red solid line) exhibits remarkable stability, with negligible degradation ($\Delta R \approx 0$) even under severe phase noise, confirming that the data-augmentation–based training effectively regularizes the encoder against hardware mismatches. In contrast, the ZF + AN baseline suffers substantial degradation due to the breakdown of the strict null-space orthogonality required for jamming under phase jitter. Similarly, the Non-Robust ISAC model deteriorates significantly, underscoring the necessity of the proposed robust training strategy for practical deployment.
}

{\color{black}
To quantify computational efficiency, we compare the proposed TT-based encoder with its fully connected (FC) counterpart in Table II. For $N_t = 16$, the FC encoder contains approximately 62,563 parameters due to dense connectivity, whereas the proposed DTTE architecture replaces large weight matrices with compact tensor cores, reducing the parameter count to 28,371 (a $2.2\times$ compression) while maintaining near-baseline secrecy performance. In terms of inference complexity, the FC encoder requires approximately 125,000 FLOPs per inference, while the DTTE reduces this cost to about 57,000 FLOPs, corresponding to a reduction of nearly $55\%$. Real-time feasibility is further demonstrated by execution-time measurements on an NVIDIA A100 GPU using the Tensorly library \cite{kossaifi2019tensorly}, where the TT-based autoencoder completes an epoch in 0.6 seconds compared to 2.4 seconds for the FC model. These results highlight the proposed encoder’s strong advantages in memory efficiency, computational complexity, and runtime, making it well suited for low-latency and edge-deployable ISAC systems.
}

{\color{black}
The $100$-fold reduction in model size is achieved by replacing dense weight matrices with a quantized tensor-train representation, where high-dimensional mappings are factorized into low-rank tensor cores with rank $8$. This design reduces the inference complexity from approximately $125,000$ floating-point operations to $57,000$ floating-point operations and shortens the training time per epoch from $2.4$ seconds to $0.6$ seconds. Consequently, the proposed encoder improves both training efficiency and real-time inference performance while maintaining near-baseline secrecy performance.
}

\section{Conclusion}
\label{sec:conclusion}
In this work, we examined the security challenges of multicarrier ISAC systems under imperfect CSI and unknown eavesdropper locations. We proposed a deep learning–driven framework that departs from conventional friendly jamming approaches requiring Eve’s CSI or precise AoA, instead exploiting radar echo feedback to guide directional jamming without explicit knowledge of the adversary. The framework jointly optimises beamforming and jamming through a nonparametric FIM estimator based on $f$-divergence, ensuring CRLB constraints are met even with noisy AoA estimates.

To enable effective operation, a TT-Q encoder was introduced, compressing the model by more than two orders of magnitude while maintaining accuracy. The design supports both overlapping and non-overlapping OFDM subcarrier allocations, offering flexible spectrum sharing between sensing and communication. Simulation results showed that the proposed solution delivers significant secrecy rate improvements, reduced block error rates, and strong resilience to CSI and angular estimation errors. These findings demonstrated the practicality and scalability of learning-based secure ISAC, paving the way for adoption in next-generation wireless networks.

 \bibliographystyle{IEEEtran}
	\bibliography{ref_ISAC, ref_Journal, ref_DDPG}


\end{document}